%% file: main.tex
\documentclass[10pt,twocolumn,letterpaper]{article}

\usepackage[pagenumbers]{cvpr} 

\input{preamble}
\input{math_commands.tex}

\theoremstyle{plain}
\newtheorem{theorem}{Theorem}[section]

\newtheorem{lemma}[theorem]{Lemma}

\theoremstyle{definition}

\theoremstyle{remark}

\definecolor{cvprblue}{rgb}{0.21,0.49,0.74}
\usepackage[pagebackref,breaklinks,colorlinks,allcolors=cvprblue]{hyperref}

\def\paperID{6} 
\def\confName{CVPR CVEU}
\def\confYear{2026}

\title{FOCUS:\\Optimal Control for Multi-Entity World Modeling in Text-to-Image Generation}

\author{\vspace{1mm}
    Eric Tillmann Bill
    \hspace{0.5cm}
    Enis Simsar
    \hspace{0.5cm}
    Thomas Hofmann
    \\
    ETH Zürich\\
    {\tt\small \href{https://ericbill21.github.io/FOCUS/}{\textbf{ericbill21.github.io/FOCUS/}}}
}

\begin{document}

\twocolumn[{%

\renewcommand\twocolumn[1][]{#1}%
\maketitle

\begin{center}
\vspace{-2mm}
\begingroup
\setlength{\parskip}{0pt}\setlength{\parindent}{0pt}

\newlength{\midsep}
\setlength{\midsep}{3mm}

\noindent
\makebox[\dimexpr(\linewidth-2\midsep)/2\relax][l]{\large\textbf{Base Models}}%
\hspace{\midsep}%
\makebox[\dimexpr(\linewidth-2\midsep)/2\relax][l]{\large\textbf{Base Models + FOCUS (Ours)}}%
\\

\includegraphics[width=\dimexpr(\linewidth-\midsep)/4\relax]{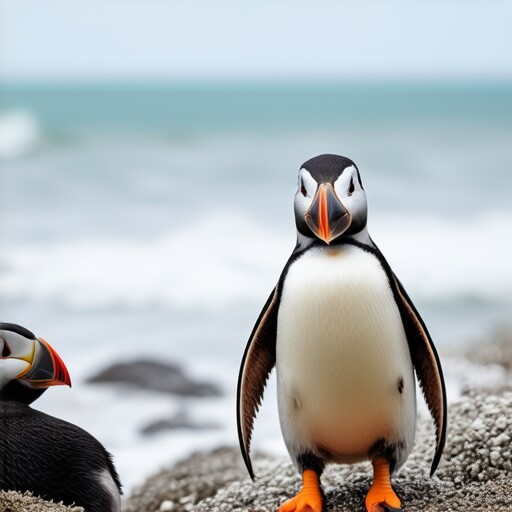}%
\includegraphics[width=\dimexpr(\linewidth-\midsep)/4\relax]{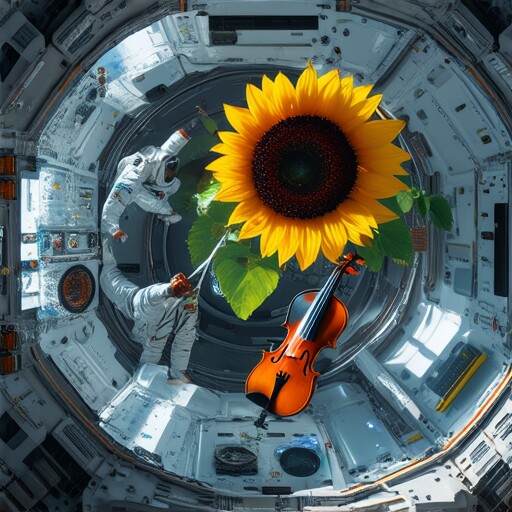}%
\hspace{\midsep}%
\includegraphics[width=\dimexpr(\linewidth-\midsep)/4\relax]{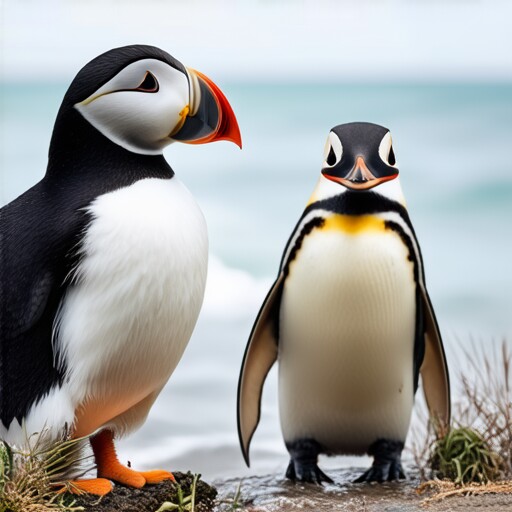}%
\includegraphics[width=\dimexpr(\linewidth-\midsep)/4\relax]{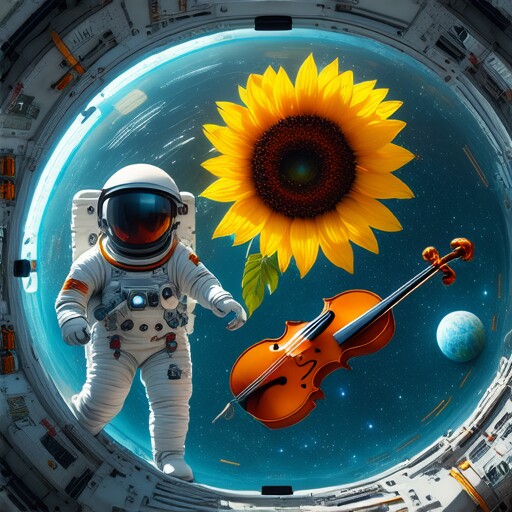}%
\\\vspace{-0.5mm}

\includegraphics[width=\dimexpr(\linewidth-\midsep)/6\relax]{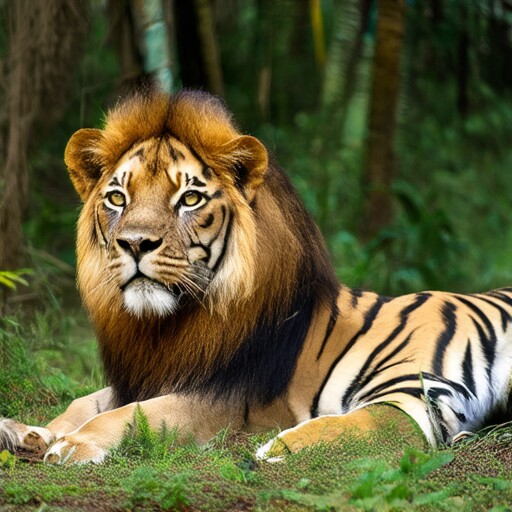}%
\includegraphics[width=\dimexpr(\linewidth-\midsep)/6\relax]{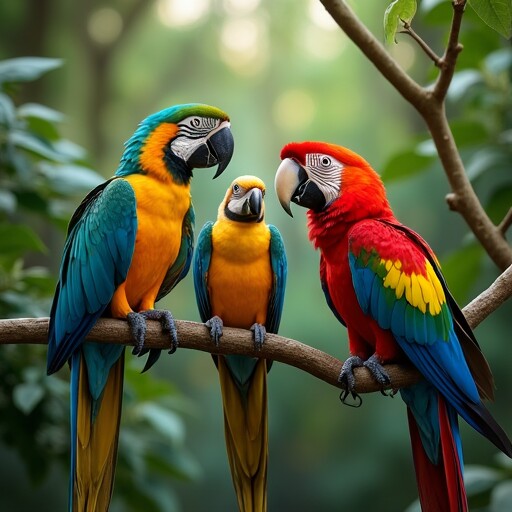}%
\includegraphics[width=\dimexpr(\linewidth-\midsep)/6\relax]{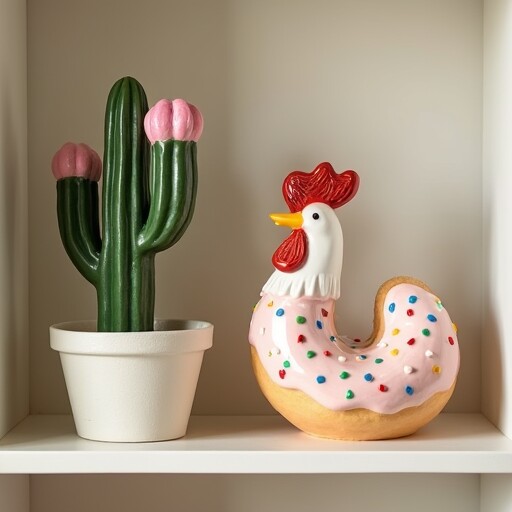}%
\hspace{\midsep}%
\includegraphics[width=\dimexpr(\linewidth-\midsep)/6\relax]{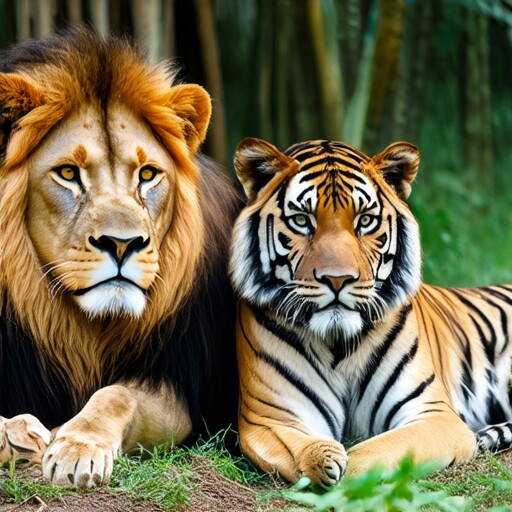}%
\includegraphics[width=\dimexpr(\linewidth-\midsep)/6\relax]{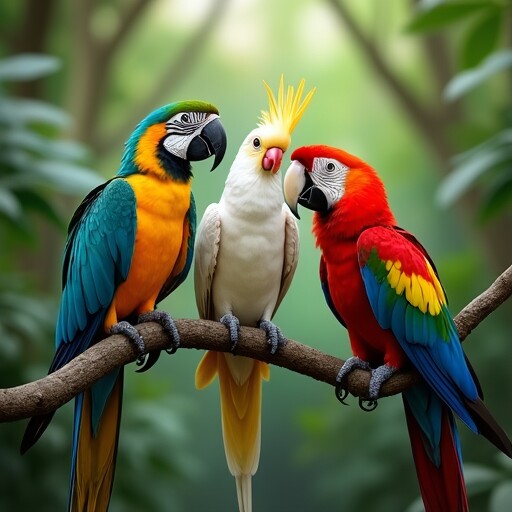}%
\includegraphics[width=\dimexpr(\linewidth-\midsep)/6\relax]{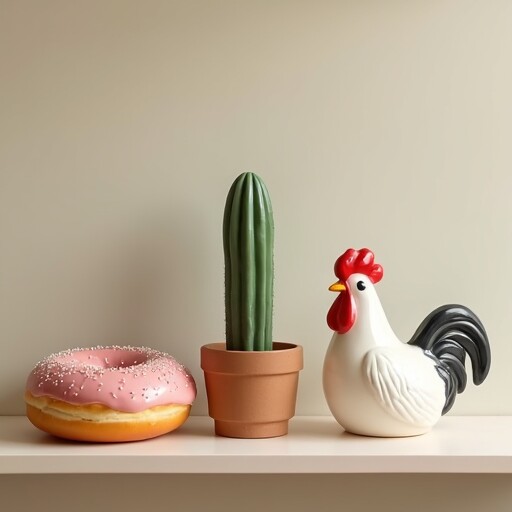}%
\\\vspace{-0.5mm}


\includegraphics[width=\dimexpr(\linewidth-1.1\midsep)/8\relax]{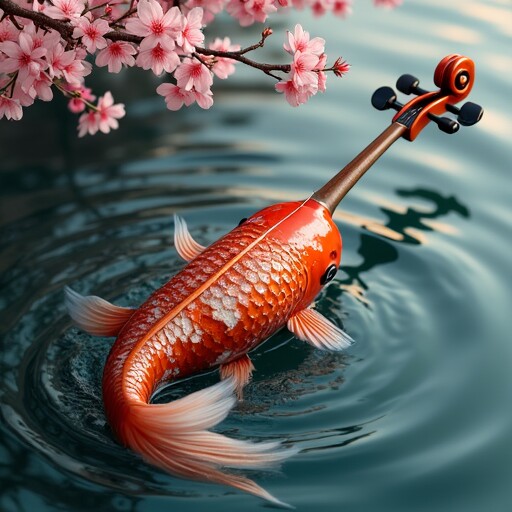}%
\includegraphics[width=\dimexpr(\linewidth-1.1\midsep)/8\relax]{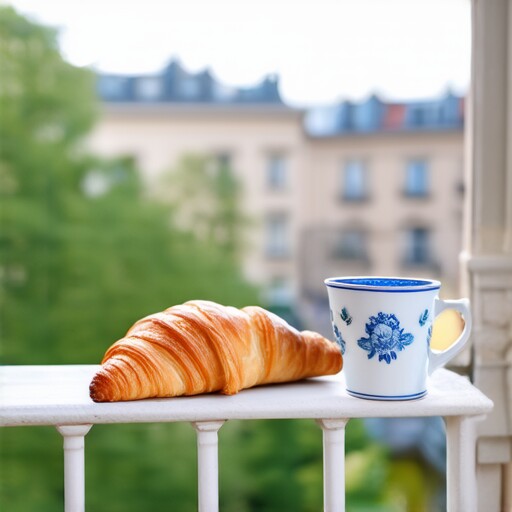}%
\includegraphics[width=\dimexpr(\linewidth-1.1\midsep)/8\relax]{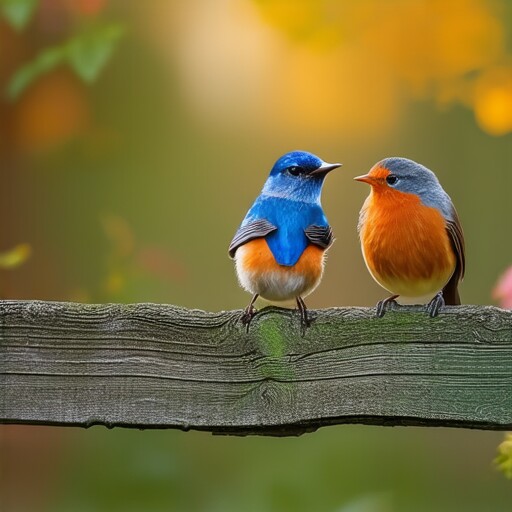}%
\includegraphics[width=\dimexpr(\linewidth-1.1\midsep)/8\relax]{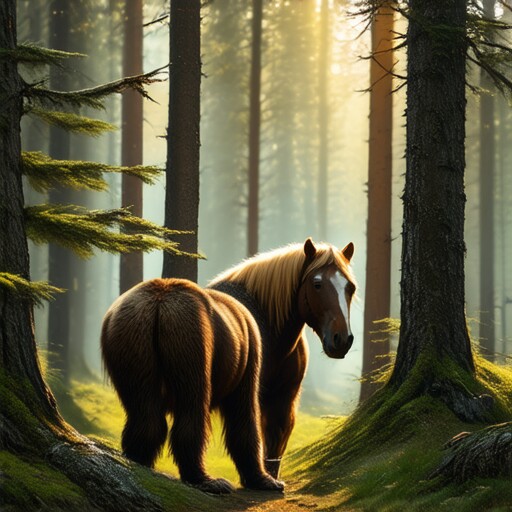}%
\hspace{\midsep}%
\includegraphics[width=\dimexpr(\linewidth-1.1\midsep)/8\relax]{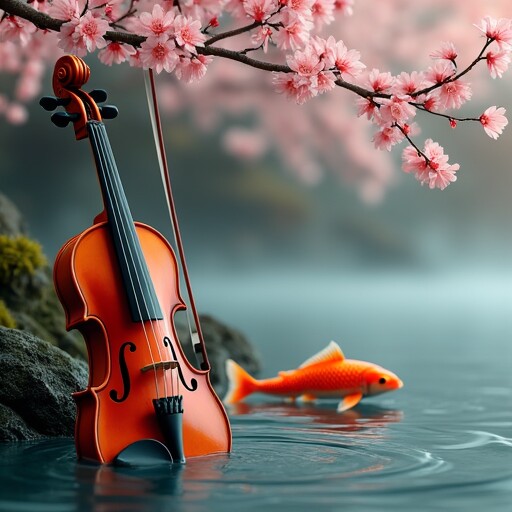}%
\includegraphics[width=\dimexpr(\linewidth-1.1\midsep)/8\relax]{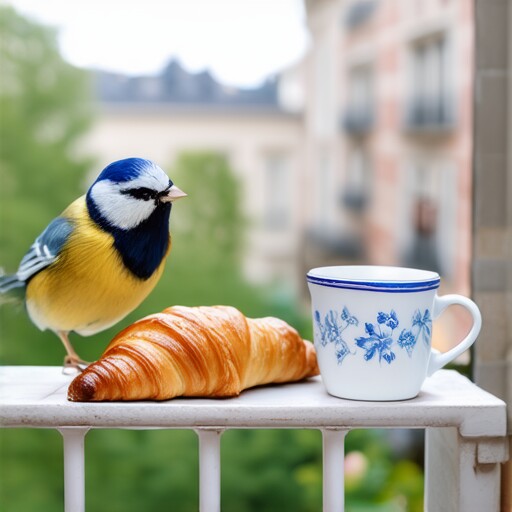}%
\includegraphics[width=\dimexpr(\linewidth-1.1\midsep)/8\relax]{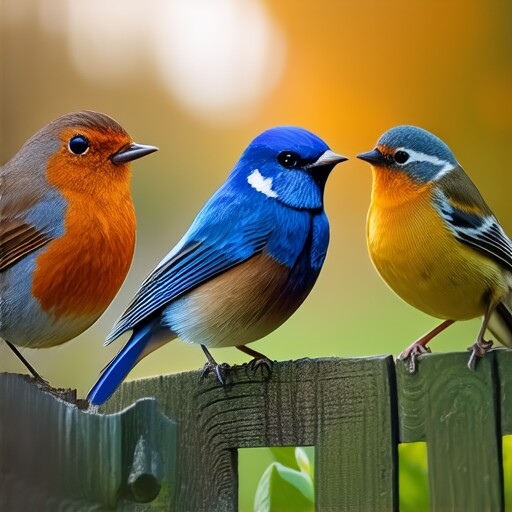}%
\includegraphics[width=\dimexpr(\linewidth-1.1\midsep)/8\relax]{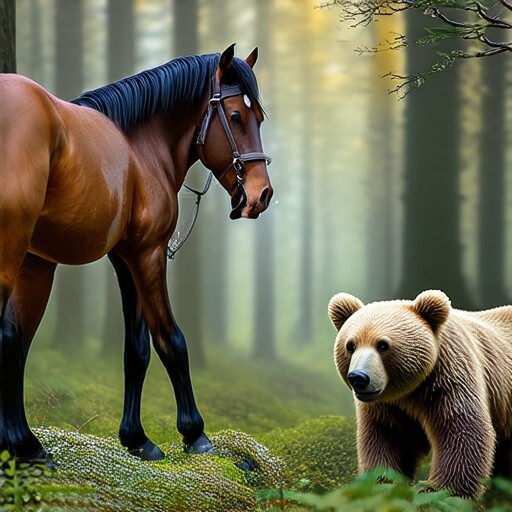}%
\endgroup
\vspace{-1mm}

\captionof{figure}{\textbf{Optimal control improves coherent multi-subject generation in flow matching models.} Using \textsc{FOCUS} at test time or via fine-tuning yields faithful compositions with correct attribute binding, minimal leakage, and no omissions, while preserving base style.}
\label{fig:teaser}
\vspace{2mm}
\end{center}
}]

\begin{abstract}\setlength{\parskip}{0pt}
Text-to-image (T2I) models excel on single-entity prompts but struggle with multi-entity scenes, often exhibiting attribute leakage, identity entanglement, and subject omissions. We present a principled theoretical framework that steers sampling toward multi-subject fidelity by casting flow matching (FM) as stochastic optimal control (SOC), yielding a single hyperparameter controlled trade-off between fidelity and object-centric state separation / binding consistency. Within this framework, we derive two architecture-agnostic algorithms: (i) a training-free test-time controller that perturbs the base velocity with a single-pass update, and (ii) \emph{Adjoint Matching}, a lightweight fine-tuning rule that regresses a control network to a backward adjoint signal. The same formulation unifies prior attention heuristics, extends to diffusion models via a flow–diffusion correspondence, and provides the first fine-tuning route explicitly designed for multi-subject fidelity. In addition, we also introduce \textsc{FOCUS} (Flow Optimal Control for Unentangled Subjects), a probabilistic attention-binding objective compatible with both algorithms. Empirically, on Stable Diffusion~3.5 and FLUX.1, both algorithms consistently improve multi-subject alignment while maintaining base-model style; test-time control runs efficiently on commodity GPUs, and fine-tuned models generalize to unseen prompts. 
\end{abstract}

\vspace{-3\baselineskip}
\input{sections/intro}
\input{sections/preliminary}
\input{sections/method}
\input{sections/related_work}

\input{sections/emprical_results}
\input{sections/discussion}
\newpage

{
    \small
    \bibliographystyle{ieeenat_fullname}
    \bibliography{references}
}

\clearpage
\appendix
\setcounter{page}{1}
\onecolumn
\maketitlesupplementaryone
\section*{Supplementary Material Contents}
\startcontents[sections]
\printcontents[sections]{l}{1}{\setcounter{tocdepth}{2}}

{%
  \let\clearpage\relax
  \let\newpage\relax
  \maketitlesupplementary
    \onecolumn
  \startcontents[sections]
  \printcontents[sections]{l}{1}{\setcounter{tocdepth}{2}}%
}
\onecolumn
\newpage
\input{sections/appendix}

\end{document}

%% file: preamble.tex
\newcommand{\medtit}[1]{\medbreak\noindent\textbf{#1}}

\usepackage{amsmath,amsfonts,bm}
\usepackage{amssymb}
\usepackage{mathtools}
\usepackage{amsthm}
\usepackage{nicefrac}

\usepackage{url}

\usepackage[ruled,vlined,noend]{algorithm2e}
\usepackage{subcaption}
\DontPrintSemicolon
\SetKwInput{KwIn}{Input}
\SetKwInput{KwOut}{Output}

\usepackage{caption}
\usepackage{graphicx}
\usepackage{float}

\usepackage{booktabs}
\usepackage{tabularx}
\usepackage{multirow}

\usepackage{wrapfig}

\usepackage{tikz}
\usepackage{pgfplots}
\usepackage{pgfplotstable}
\pgfplotsset{compat=1.17}
\usepackage{subcaption}
\usepackage{capt-of} 

\usepackage{titletoc}
\usepackage[page,header]{appendix}

\usepackage[table]{xcolor}
\definecolor{gold}{RGB}{255,215,0}
\definecolor{silver}{RGB}{192,192,192}
\definecolor{bronze}{RGB}{205,127,50}

\definecolor{PosColor}{RGB}{0,112,192}
\definecolor{NegColor}{RGB}{192,0,0}   

\newcommand{\HEXAND}[2]{%
    \includegraphics[width=0.2\linewidth]{figures/#1/base/img_#2.jpg} &
    \includegraphics[width=0.2\linewidth]{figures/#1/attend_and_excite/img_#2.jpg} &
    \includegraphics[width=0.2\linewidth]{figures/#1/conform/img_#2.jpg} &
    \includegraphics[width=0.2\linewidth]{figures/#1/divide_and_bind/img_#2.jpg} &
    \includegraphics[width=0.2\linewidth]{figures/#1/scg/img_#2.jpg} &
    \includegraphics[width=0.2\linewidth]{figures/#1/focus/img_#2.jpg}
    \vspace{-0.5mm}
}

\newcommand{\QUINTAND}[2]{%
    \includegraphics[width=0.2\linewidth]{figures/#1/base/img_#2.jpg} &
    \includegraphics[width=0.2\linewidth]{figures/#1/attend_and_excite/img_#2.jpg} &
    \includegraphics[width=0.2\linewidth]{figures/#1/conform/img_#2.jpg} &
    \includegraphics[width=0.2\linewidth]{figures/#1/divide_and_bind/img_#2.jpg} &
    \includegraphics[width=0.2\linewidth]{figures/#1/focus/img_#2.jpg}
    \vspace{-0.5mm}
}

\newcommand{\TRIPLE}[2]{%
    \includegraphics[width=0.32\linewidth]{figures/#1/base/img_#2.jpg} &
    \includegraphics[width=0.32\linewidth]{figures/#1/scg/img_#2.jpg} &
    \includegraphics[width=0.32\linewidth]{figures/#1/focus/img_#2.jpg} &
    \vspace{-0.5mm}
}

\input{sections/tables}

%% file: math_commands.tex
\def\eqref#1{equation~\ref{#1}}

\def\1{\bm{1}}

\def\vm{{\bm{m}}}

\def\vp{{\bm{p}}}
\def\vq{{\bm{q}}}

\def\vs{{\bm{s}}}

\def\evp{{p}}
\def\evq{{q}}

\def\mI{{\bm{I}}}

\DeclareMathAlphabet{\mathsfit}{\encodingdefault}{\sfdefault}{m}{sl}
\SetMathAlphabet{\mathsfit}{bold}{\encodingdefault}{\sfdefault}{bx}{n}

\newcommand{\E}{\mathbb{E}}

\newcommand{\R}{\mathbb{R}}



%% file: sections/intro.tex
\section{Introduction}
Text-to-image (T2I) models have made substantial progress in visual fidelity and prompt adherence, yet they remain brittle on \emph{multi-subject} prompts. Typical failure modes include attribute leakage, identity entanglement, and subject omission \cite{chefer_attend-and-excite_2023,liu_compositional_2023,bar-tal_multidiffusion_2023,dahary_be_2024}. From a \emph{world modeling} perspective, these errors reflect weak object-centric representations of \emph{which entities} are present and \emph{which attributes} belong to each, so constraints become non-local (\eg{,} changing ``the red hat'' unintentionally alters another subject). This limits applications such as story illustration, multi-panel composition, and scientific communication, where identity and attribute binding are essential.

A unifying theoretical perspective on modern T2I models is \emph{flow matching} (FM), which parameterizes generation as a time-dependent flow from a base distribution to the data distribution via a learned vector field \cite{lipman_flow_2023,liu_flow_2022,albergo_stochastic_2023}. This framework encompasses recent large-scale systems such as Stable Diffusion~3.5~\cite{esser_scaling_2024}, FLUX~\cite{labs_flux1_2025}, and earlier denoising-diffusion architectures such as Stable Diffusion~1.5~\cite{rombach_high-resolution_2022} and Stable Diffusion~XL~\cite{podell_sdxl_2023}. FM thus provides a common language for studying how generative models allocate attributes to subjects and maintain scene coherence across architectural families. We leverage this common ground to analyze and improve multi-subject fidelity in FM models. 

Prior work has attempted to mitigate multi-subject failures through \emph{test-time} heuristics that manipulate cross-attention \cite{meral_conform_2023,qiu2025self} or adjust guidance \cite{feng_training-free_2023}, including token amplification \cite{chefer_attend-and-excite_2023}, constraint-based binding \cite{li_divide_2024}, and structure-aware attention editing \cite{hertz_prompt--prompt_2022,dahary_be_2024}. While often effective in specific settings, these methods are heuristic and lack a unifying optimization objective, making it unclear when and why they succeed. They also require several hyperparameters and are primarily developed for Stable Diffusion~1.x backbones, which limits their portability to modern FM models.

In this work, we connect guidance-style latent updates to \emph{stochastic optimal control} (SOC). Under standard approximations, SOC yields a gradient-style control on the sampler dynamics that is closely related to prior latent-update guidance, but additionally provides a single objective with an explicit trade-off between base fidelity and entity disentanglement. Building on this view, we propose \textsc{FOCUS} (Flow Optimal Control for Unentangled Subjects), a probabilistic attention-binding objective that measures multi-entity entanglement and encourages factorized subject representations. The resulting objective supports two practical instantiations: (i) a lightweight test-time controller that steers sampling without retraining, and (ii) a fine-tuning procedure that learns a control network via Adjoint Matching \cite{domingo-enrich_adjoint_2025}. 

Empirically, both instantiations of our objective---test-time control and fine-tuning---improve multi-subject fidelity on modern FM models (SD~3.5, FLUX) and transfer to older diffusion backbone (SD~1.5) under matched compute budgets. Test-time control provides consistent improvements with low overhead on commodity GPUs, while fine-tuning learns an amortized controller that maintains gains on unseen prompts. We validate these trends with automatic metrics and a human preference study.

In summary, our contributions are threefold:
\begin{enumerate}[label=(\roman*),leftmargin=0.75cm]
    \item \textbf{\textsc{FOCUS}:} a divergence-based, probabilistic attention-binding objective for multi-subject disentanglement.
    \item \textbf{Two instantiations:} (a) a training-free test-time controller and (b) an Adjoint Matching \cite{domingo-enrich_adjoint_2025} fine-tuning route that learns a lightweight control network.
    \item \textbf{SOC perspective:} an SOC interpretation of latent-update guidance for FM samplers, exposing an explicit trade-off knob and improving portability across architectures.
\end{enumerate}
We provide our code, curated dataset, and model checkpoints at {\tt\href{https://ericbill21.github.io/FOCUS/}{ericbill21.github.io/FOCUS/}}.

%% file: sections/preliminary.tex
\section{Preliminaries}
Flow Matching (FM) \cite{lipman_flow_2023,liu_flow_2022,albergo_stochastic_2023} trains a time–dependent vector field $v_\theta:\R^d\times[0,1]\to\R^d$ that transports a base distribution $\pi_0$ (\eg{,} $\mathcal N(0,I)$) to a target distribution $\pi_1$ (\eg{,} $P_\text{data}$). Given a \emph{reference path} $(\overline{X}_t)_{t\in[0,1]}$ with $\overline{X}_0 \sim \pi_0$ and $\overline{X}_1 \sim \pi_1$, FM regresses the \emph{conditional velocity}
\begin{align}
    u_t(\overline{X}_t \mid \overline{X}_0,\overline X_1):=\frac{d}{dt}\overline{X}_t
\end{align}
so that $v_\theta(x,t)$ matches $\E[u_t\mid \overline X_t=x]$ \cite{lipman_flow_2023}.

\medtit{Reference paths.} A standard choice for reference paths is linear interpolation. For given $\overline{X}_0, \overline{X}_1$ we define $\overline{X}_t$ as
\begin{align*}
    \overline{X}_t = \beta_t \overline{X}_0 + \alpha_t \overline{X}_1,\quad\text{s.t.}\quad \alpha_0=\beta_1=0, \alpha_1=\beta_0=1
\end{align*}    
where $(\alpha_t, \beta_t)_{t\in[0,1]}$ is a differentiable scheduler with $\alpha_t$ strictly increasing, and $\beta_t$ strictly decreasing. The pathwise derivative is then $u_t(\overline{X}_t \mid \overline{X}_0,\overline{X}_1)=\dot\beta_t \overline{X}_0 + \dot{\alpha}_t \overline{X}_1$.\footnote{Over-dot denotes the time derivative, \ie{,} $\dot{x}_t=\frac{d}{dt}x_t$.} A widely used instance is \emph{rectified flow} (RF), also known as optimal transport, with $\alpha_t=t$ and $\beta_t=1-t$ \cite{liu_flow_2022}.

\medtit{Training objective.}
FM is trained with the \emph{conditional flow matching} loss $\mathcal L_{\mathrm{CFM}}(\theta)$\cite{lipman_flow_2023} defined as follows
\begin{align}
    \E_{t\sim\mathcal U[0,1]}
    \E_{\substack{\overline{X}_0 \sim \pi_0\\ \overline{X}_1 \sim \pi_1}}
    \left[ \big\| u_t(\overline{X}_t \mid \overline{X}_0, \overline{X}_1) - v_\theta(\overline{X}_t,t) \big\|_2^2 \right],\nonumber
\end{align}
which regresses the pathwise velocity toward its conditional mean at uniformly sampled times.

\medtit{Sampling.}
After training, sample $X_0\sim\pi_0$ and evolve the learned flow by solving the ODE
\begin{align}
    dX_t \;=\; v_\theta(X_t,t)\,dt,
\end{align}
which produces a path $(X_t)_{t\in[0,1]}$ whose marginals match those of the reference path $(\overline{X}_t)_{t\in[0,1]}$ under standard existence–uniqueness conditions; in particular $X_1\sim\pi_1$ \cite{lipman_flow_2023}. More generally, FM admits a stochastic formulation \cite{domingo-enrich_adjoint_2025} in which the drift is augmented by a arbitrary schedule-dependent correction with diffusion coefficient $\sigma(t)\ge 0$:
\begin{align}
    d X_t = b(X_t, t) dt + \sigma(t)\, dB_t, \label{eq:sde}
\end{align}
where $(B_t)_{t\geq 0}$ is standard Brownian motion in $\R^d$, and $b(X_t,t)$ is referred to as the (base) drift, defined as
\begin{align}
    v_\theta(X_t,t) 
      + \frac{\sigma(t)^2}{2\beta_t \left( \frac{\dot{\alpha}_t}{\alpha_t}\beta_t - \dot{\beta}_t \right) }
         \left( v_\theta(X_t,t) - \frac{\dot{\alpha}_t}{\alpha_t}X_t \right).
\end{align}
Setting $\sigma\equiv 0$ recovers the deterministic ODE.

\medtit{Connection to denoising diffusion.}
Classical denoising diffusion models arise as special cases of FM when their discrete procedures are lifted to continuous time; refer to the Supplementary Material for full derivation.

%% file: sections/method.tex
\section{Methodology}
We formulate disentanglement as optimal control over flow-matching dynamics, derive test-time and fine-tuned controllers, and introduce a probabilistic loss \textsc{FOCUS}.

\subsection{Stochastic Optimal Control}
Our goal is to reduce multi-subject entanglement while remaining close to the base model. To this end, we introduce a small \emph{control}
$u:\R^d\times[0,1]\to\R^d$ into the drift and pose generation as a quadratic, control-affine SOC problem:
\begin{align}
    \min_{u\in\mathcal U} \E &\left[ \int_0^1 \frac{1}{2} \|u(X_t^{u}, t)\|_2^2 + f(X_t^{u},t) dt + g(X_1^{u})\right],\label{eq:soc}\\
    \shortintertext{s.t.}
    \begin{split}
    dX_t^{u} &= \left( b(X_t^{u},t) + \sigma(t) u(X_t^{u},t) \right) dt
    + \sigma(t) dB_t, \label{eq:soc_st}\\
    X_0^{u} &\sim \pi_0,
    \end{split}
\end{align}
where $X_t^{u}$ is the latent state, $b$ is the base FM drift, $\sigma(t)\ge 0$ is a scalar diffusion schedule, and $(B_t)_{t\in[0,1]}$ is Brownian motion. The running cost $f:\R^d\times[0,1]\to\R$ will measure subject entanglement (e.g. $f \equiv \mathrm{FOCUS}$), and we set the terminal cost $g\equiv 0$ in all derivations and experiments.

For control-affine dynamics with $\ell(x,u,t)=\tfrac{1}{2}\|u\|_2^2+f(x,t)$, the Hamiltonian $\mathcal{H}(x,u,a,t)$ of the SOC is defined as follows
\begin{align}
     \frac{1}{2} \|u\|_2^2 + f(x,t) + a^\top \left(b(x,t) + \sigma(t) u\right),
\end{align}
where $a(t) \in \R^d$ is the co-state (adjoint). Since $\mathcal{H}$ is strictly convex in $u$, the first-order optimality condition yields 
\begin{align}
    u_t^\star = -\sigma(t)a(t), \label{eq:ham_opt}
\end{align}
with adjoint dynamics
\begin{align}
\begin{split}
    \frac{d}{dt} a(t) &= -\left[ \nabla_X b(X_t^{u},t)^\top a(t) + \nabla_X f(X_t^{u}, t)\right],\\
    a(1)&=\nabla_X g(X_1^{u}). \label{eq:adjoint}
\end{split}
\end{align}

\subsection{On-the-fly disentanglement (test-time control)}
At inference, we solve \cref{eq:soc} \emph{per trajectory} with frozen model parameters. The idea is to compute $u^\star_t$ on-the-fly and steer the sampling process at each timestep $t$. Directly computing $u_t^\star$ requires the adjoint state $a(t)$ in \cref{eq:ham_opt}, which is defined along the \emph{controlled} path via \cref{eq:adjoint}. This is impractical because $a(t)$ depends on the terminal condition $a(1)=\nabla_X g(X^{u}_1)$, which depends on the endpoint $X^{u}_1$, which in turn depends on the future segment $(X_\tau^{u})_{\tau\in[t,1]}$ ; coupling a backward solve to the forward pass at every step.
To obtain a \emph{single-pass} controller, we approximate $a(t)$ locally at the current state. Concretely, we linearize \cref{eq:adjoint} around $X_t^{u}$, freeze $\nabla_X b\approx 0$, and treat the future state as locally constant:
\begin{align}
    a(t) \approx \int_t^1 \nabla_X f(X_t^u, \tau) d\tau \approx (1-t) \nabla_X f(X_t^u, t),
\end{align}
where the last step uses a left-Riemann approximation. Substituting into \cref{eq:ham_opt} yields the instantaneous control
\begin{align}
    u^\star_t \approx -\sigma(t) (1-t) \nabla_X f(X_t^u, t).
\end{align}
The approximation $\nabla_X b\approx 0$ is common in online control settings \cite{havens_adjoint_2025}.

\medtit{Velocity reparameterization.}
Let $v_{\text{base}}$ denote the base FM velocity parameterization. For fine-tuning we adopt the \emph{memoryless} diffusion schedule of \citet{domingo-enrich_adjoint_2025}, which makes the stochastic interpolant endpoints independent $X_0 \perp X_1$ and yields the closed-form diffusion coefficient
\begin{align}
    \sigma_{\mathrm{mem}}(t)
    = \sqrt{2\,\beta_t\!\left(\frac{\dot\alpha_t}{\alpha_t}\beta_t - \dot\beta_t\right)}.
    \label{eq:mem}
\end{align}
Under this schedule, the standard drift--velocity identity for the interpolant SDE \cref{eq:sde} simplifies to
$b(X_t,t)=2v_\theta(X_t,t)-\frac{\dot\alpha_t}{\alpha_t}X_t$.
Injecting control in the drift as $b \mapsto b+\sigma_{\mathrm{mem}}(t)u_t$ therefore corresponds to a \emph{direct additive} modification of the velocity,
$v_\theta \mapsto v_\theta + \tfrac{1}{2}\sigma_{\mathrm{mem}}(t)u_t$.
Consequently, the controlled velocity passed to the sampler is
\begin{align}
    v_t^\star
    &= v_{\text{base}}(X_t,t) + \frac{\sigma_{\mathrm{mem}}(t)}{2}\,u_t^\star \\
    &\approx v_{\text{base}}(X_t,t) - \underbrace{\frac{\sigma_{\mathrm{mem}}^2(t)}{2}(1-t)}_{\eta_{\mathrm{mem}}(t)}\,\nabla_X f(X_t,t), \label{eq:sde_update}
\end{align}
\ie{,} a standard latent-gradient guidance term with an explicit time-dependent step size $\eta_{\mathrm{mem}}(t)$ implied by the SOC approximation. This form can be plugged into any off-the-shelf SDE solver without modifying the integrator.


\medtit{Deterministic sampling at test time.}
At inference, we can apply the same guidance term with deterministic ODE samplers commonly used in T2I systems ($\sigma \equiv 0$). In this case, the controlled velocity reduces to
\begin{align}
    v_t^\star \;=\; v_{\text{base}}(X_t,t) - \eta(t)\,\nabla_X f(X_t,t), \label{eq:ode_update}
\end{align}
with the same objective $f$ (\eg{,} \textsc{FOCUS}) and time schedule $\eta(t)$ as in the stochastic setting. Full derivations are provided in the Supplementary Material.

\subsection{Fine-tuning (amortizing test-time guidance)}
Our goal is to learn a lightweight control network $u_\theta(X_t,t)$ that \emph{amortizes} the per-step guidance signal used at test time, so that disentanglement can be obtained with a single forward pass and reduced latency.

\medtit{Training recipe.}
We adopt \emph{Adjoint Matching} (AM) from \citet{domingo-enrich_adjoint_2025} as an efficient way to train the control network $u_\theta$.
Directly solving \cref{eq:adjoint} during training is prohibitive because the adjoint $a(t)$ depends on the controlled path $X_t^u$. AM avoids differentiating through controlled trajectories by regressing $u_\theta$ to a cheaper \emph{lean adjoint} $\tilde a(t)$ computed along \emph{frozen} forward trajectories $(X_t)_{t\in[0,1]}$, while dropping $u$-dependent Jacobian terms:
\begin{align}
    \begin{split}
        \frac{d}{dt} \tilde{a}(t) &= -\left[\nabla_X b(X_t,t)^\top \tilde{a}(t) + \nabla_X f(X_t,t)\right], \\
        \tilde{a}(1)&=\nabla_X g(X_1). \label{eq:lean_adjoint}
    \end{split}
\end{align}
This produces a training target aligned with the SOC optimal control $u_t^\star=-\sigma(t)a(t)$, with $a(t)$ replaced by the tractable surrogate $\tilde a(t)$.

\medtit{Memoryless training.}
To improve generalization beyond the specific sampled trajectories and stabilize the regression target, we utilize the \emph{memoryless} schedule $\sigma_{\mathrm{mem}}$, which makes the endpoints independent by design $X_0 \perp X_1$ \cite{domingo-enrich_adjoint_2025}.

\medtit{Training objective.}
Each iteration proceeds as follows: (i) sample forward trajectories $(X_t)_{t\in[0,1]}$ under $\sigma_{\mathrm{mem}}$ with the base model frozen via \cref{eq:sde}; (ii) integrate the lean adjoint $(\tilde a(t))_{t\in[0,1]}$ backwards with \cref{eq:lean_adjoint}; (iii) regress the control toward the stationary target $-\sigma_{\mathrm{mem}}(t)\tilde a(t)$ by minimizing
\begin{align}
        \mathcal{L}_\mathrm{AM}(\theta) := \frac{1}{2} \int_0^1 \|u_\theta(X_t, t) + \sigma_{\mathrm{mem}}(t)\tilde{a}(t)\|^2_2 dt.
\end{align}
The memoryless schedule is only required during \emph{fine-tuning}. At inference, $\sigma$ can be set to $0$, enabling fast off-the-shelf ODE samplers while applying the learned controller.

\subsection{Loss for Multi-subject Disentanglement}
\begin{figure}
    \centering
    \begin{subfigure}[t]{0.3\linewidth}
        \centering
        \captionsetup{labelformat=empty}
        \caption{\textbf{Dachshund}}
        \includegraphics[width=\linewidth]{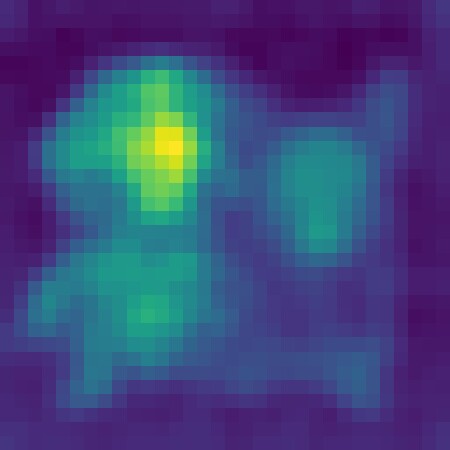}
    \end{subfigure}
    \hfill
    \begin{subfigure}[t]{0.3\linewidth}
        \centering
        \captionsetup{labelformat=empty}
        \caption{\textbf{Corgi}}
        \includegraphics[width=\linewidth]{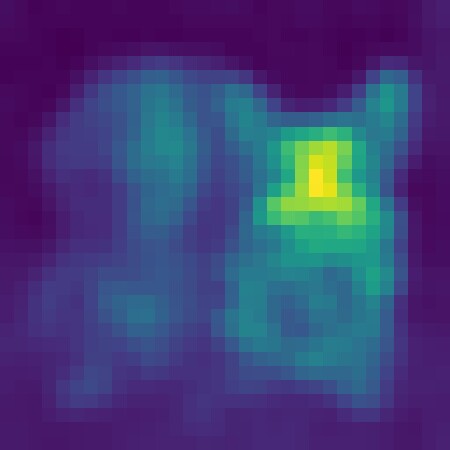}
    \end{subfigure}
    \hfill
    \begin{subfigure}[t]{0.3\linewidth}
        \centering
        \captionsetup{labelformat=empty}
        \caption{\textbf{Image}}
        \includegraphics[width=\linewidth]{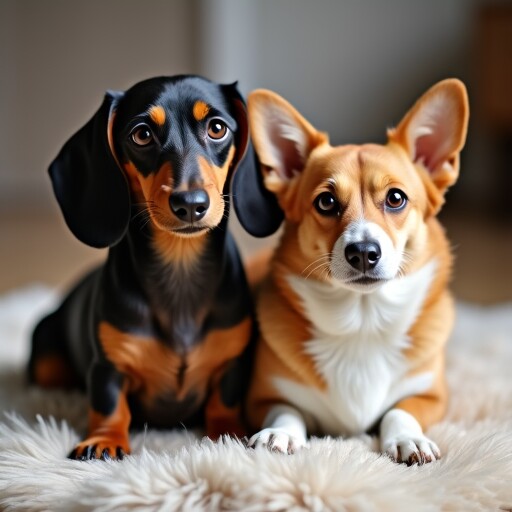}
    \end{subfigure}\\
    \vspace{1mm}
    \textit{\small ``A dachshund and a corgi sitting together on a cozy rug''}
    \vspace{-2mm}
    \caption{Extracted cross-attention maps from FLUX.1 [dev].}
    \label{fig:attn_example}
\end{figure}

To instantiate the running cost $f$ in our SOC objective, we leverage the fact that T2I backbones compute cross-attention from image-space queries to text tokens. Empirically, these token-wise cross-attention maps correlate with the eventual spatial placement of the corresponding entities~\cite{hertz_prompt--prompt_2022,chefer_attend-and-excite_2023}. This provides a natural signal to diagnose and mitigate subject entanglement \emph{during} generation by measuring spatial interactions among \emph{subject-specific} attention maps; see \cref{fig:attn_example} for an example of extracted attention maps.

\begin{figure*}[t]
      \centering
      {\setlength{\tabcolsep}{.3mm}%
      \resizebox{\linewidth}{!}{%
        \begin{tabular}{@{}ccccccc@{}}%
            & \textbf{Base} & \textbf{Attend\&Excite} & \textbf{CONFORM} & \textbf{Divide\&Bind} & \textbf{Self-Cross Guidance} & \textbf{\textsc{FOCUS} (Ours)} \\
            \rotatebox{90}{\parbox{0.2\linewidth}{\centering \textbf{SD 3.5}}} & \HEXAND{SD3-OTF}{3}\\
            & \multicolumn{6}{c}{\textit{\small``An astronaut, a violin, and a sunflower floating inside a space station''}}\vspace{1mm}\\
            \rotatebox{90}{\parbox{0.2\linewidth}{\centering \textbf{FLUX.1}}} & \HEXAND{FLUX-OTF}{1}\\
            & \multicolumn{6}{c}{\textit{\small``A swan, a goose, and a duck drifting past lily pads''}}\\
        \end{tabular}%
    }}
    \caption{Qualitative results with test-time control on Stable Diffusion 3.5 and FLUX.1. The sampling process is consistent for all heuristics and each heuristic is shown for its optimal $\lambda$ scale. Additional examples are provided in the Supplementary Material.}
    \label{fig:otf}\vspace{-4mm}
\end{figure*}

\medtit{Attention as a probabilistic signal.}
Existing attention-based methods \cite{meral_conform_2023,qiu2025self,chefer_attend-and-excite_2023,li_divide_2024} typically define heuristic costs on cross-attention to encourage subject separation, \eg{,} cosine-similarity or activation differences (see \cref{sec:related}). In all modern T2I backbones, however, each cross-attention map arises from a softmax and thus forms a \emph{probability distribution} over spatial locations. Our SOC framework therefore treats the collection of attention maps as a set of spatial distributions and defines the running cost $f$ directly on these distributions.

\medtit{\textsc{FOCUS}.}
We require an entanglement loss that operates on attention as probability distributions, scales to an arbitrary number of subjects, and captures both \emph{within-subject} consistency and \emph{between-subject} separation. This naturally suggests a symmetric divergence over sets of distributions; we use the Jensen--Shannon divergence, which meets these criteria and admits a simple normalization across set sizes (see Supplementary Material).

Let $d$ denote the number of spatial locations and let $\Delta^{d-1}$ be the probability simplex. For a finite set $P=\{\vp_1, \dots, \vp_n \} \subset \Delta^{d-1}$ of distributions, define the Jensen--Shannon divergence
\begin{align*}
    D_{\mathrm{JS}}(P) = \frac{1}{n} \sum_{i=1}^n D_{\mathrm{KL}} \left( \vp_i \middle\| \vm \right),
    \qquad
    \vm = \frac{1}{n} \sum_{j=1}^n \vp_j,
\end{align*}
with $D_{\mathrm{KL}}(\vp \| \vq) = \sum_{i=1}^d \evp_i \log \frac{\evp_i}{\evq_i}$ being the Kullback-Leibler divergence. Since $D_{\mathrm{JS}}(P) \in [0,\log n]$, we normalize by dividing with $\log n$ to obtain $\widehat{D}_{\mathrm{JS}}(P) \in [0,1]$, which makes scores comparable across different set sizes. We provide proof of this upper bound and more insights into this particular choice in the Supplementary Material.

We introduce \textsc{FOCUS} (Flow Optimal Control for Unentangled Subjects) to encourage, for each subject, \emph{unimodal, spatially localized, and nonoverlapping} attention. Let $S$ be the set of subjects in the prompt. For each subject $\vs \in S$, collect its attention maps at the current step into $P_\vs \subset \Delta^{d-1}$ (\eg{,} across layers or heads), and define the subject mean $\vm_\vs = \frac{1}{|P_\vs|} \sum_{\vp \in P_\vs} \vp$. Let $M=\{\vm_\vs \mid \vs \in S\}$ be the set of subject means. Our \textsc{FOCUS} loss combines \emph{within-subject agreement} and \emph{between-subject separation}:
\begin{align*}
    f_\textsc{FOCUS}(S) = \frac{1}{2}\left( 
        \frac{1}{|S|}\sum_{\vs \in S} \widehat{D}_\mathrm{JS}(P_\vs) \right) + \frac{1}{2}\left( 
        1 - \widehat{D}_\mathrm{JS}(M) \right)
\end{align*}
The first term penalizes dispersion within each subject’s maps (encouraging consistent binding, and for multi-encoder models such as SD~3.5, agreement across encoders). The second term rewards separation among subjects by maximizing divergence between their mean attention distributions. By construction, $\mathrm{focus}\in[0,1]$: $0$ indicates perfect disentanglement (low intra-subject dispersion and maximal inter-subject separation), while larger values indicate greater entanglement.

%% file: sections/related_work.tex
\section{Related Work}\label{sec:related}
We review approaches to \emph{multi-subject} T2I generation. We first cover \emph{training-free} attention-space interventions that operate at inference time. We then discuss methods that enforce \emph{regional/layout} constraints or combine multiple diffusion paths. Finally, we survey \emph{training-time} objectives that strengthen subject–attribute binding.

\begin{figure*}[t]
      \centering
      {\setlength{\tabcolsep}{.3mm}%
      \resizebox{\linewidth}{!}{%
        \begin{tabular}{@{}ccccccc@{}}%
            & \textbf{Base} & \textbf{Attend\&Excite} & \textbf{CONFORM} & \textbf{Divide\&Bind} &
            \textbf{Self-Cross Guidance} & \textbf{\textsc{FOCUS} (Ours)} \\
            \rotatebox{90}{\parbox{0.2\linewidth}{\centering \textbf{SD 3.5}}} & \HEXAND{SD3-FINE}{13}\\
            & \multicolumn{6}{c}{\textit{\small``A flamingo, a yoga mat, and a gramophone on a rooftop at sunset''}}\vspace{1mm}\\
            \rotatebox{90}{\parbox{0.2\linewidth}{\centering \textbf{FLUX.1}}} & \HEXAND{FLUX-FINE}{2}\\
            & \multicolumn{6}{c}{\textit{\small``A macaw, a cockatoo, and an Amazon parrot perched on a jungle vine''}}\\
        \end{tabular}%
    }}
    \caption{Qualitative results after fine-tuning Stable Diffusion 3.5 and FLUX.1. The sampling process is consistent for all heuristics and each heuristic is shown for its optimal $\lambda$ scale. Additional examples are provided in the Supplementary Material.}
    \label{fig:fine}
\end{figure*}

\medtit{Attention--space interventions (training--free).}
A large family of training-free methods improves multi-subject fidelity by optimizing a per-step attention objective during sampling. Most can be expressed as \emph{latent gradient guidance}: at timestep $t$, define a cost $f$ on attention maps (or related internal signals) and apply an update of the form
\begin{align}    
    X_t^\star \leftarrow X_t - \eta(t) \nabla_{X_t} f(X_t, t), \label{eq:latent_grad_update}
\end{align}
before the next model evaluation. Methods in this class differ primarily in the choice of $f$ (coverage, separation, binding, or structure) and in the guidance schedule $\eta(t)$ (inner steps, step size, and which layers/blocks are optimized).

The update above is typically applied as an \emph{inner-loop} modification of the latent, after which the sampler takes its usual step. Our SOC view makes this relationship explicit and yields a guidance term at the \emph{velocity} (or noise) level with a principled time schedule.
In particular, when using an Euler ODE solver update $X_{t+h} = X_t + h v_\theta(X_t,t)$ with step size $h$, substituting the guided velocity $v_t^\star$ from \cref{eq:ode_update} yields
\begin{align}
    X_{t+h}^\star = \underbrace{\left(X_t + h v_\theta(X_t,t) \right)}_{:= X_{t+h}} -  h \eta(t) \nabla_{X_t} f(X_t, t)
\end{align}
resembling the latent gradient guidance of \cref{eq:latent_grad_update}. 

Notable heuristics include \emph{Attend\&Excite}, which amplifies token activations to improve entity coverage and reduce neglect/leakage \cite{chefer_attend-and-excite_2023}; \emph{Divide\&Bind}, which optimizes separate terms for subject coverage and attribute binding \cite{li_divide_2024}; \emph{Structured Diffusion Guidance}, which injects linguistic structure (\eg{,} dependency trees) to shape attention manipulation for multi-object composition \cite{feng_training-free_2023}; \emph{Prompt-to-Prompt}, which constrains cross-attention correspondences to preserve word--region alignments across edits \cite{hertz_prompt--prompt_2022}; \emph{CONFORM}, which uses a contrastive (InfoNCE-style) objective to separate subjects while pulling subject--attribute pairs together \cite{meral_conform_2023}; and \emph{Self-Cross Guidance}, which additionally operates on self-attention between image tokens \cite{qiu2025self}.

While effective in specific setups, these methods typically rely on method-specific schedules and tuning choices (\eg{,} inner steps, step sizes, and attention-block selection), and their behavior can vary across backbones and samplers. In contrast, our method starts from a single SOC objective at the flow-matching level, exposing a clean fidelity--disentanglement knob $\lambda$ and yielding \emph{both} a test-time controller and a fine-tuning objective. Moreover, since our formulation operates on FM velocities, it is sampler- and architecture-agnostic and allows prior attention objectives to be reused within a unified control interface on modern models.

\medtit{Regional/layout composition and multi-path fusion.}
A complementary direction constrains \emph{where} subjects appear. \emph{MultiDiffusion} fuses multiple diffusion trajectories under shared spatial constraints (\eg{,} boxes or masks), enabling faithful multi-subject placement without retraining \cite{bar-tal_multidiffusion_2023}. Related systems extend this idea to interactive, region-based workflows. \emph{GLIGEN} augments a frozen backbone with grounding layers and conditions on bounding boxes or phrases to place multiple objects precisely \cite{li_gligen_2023}. More recently, \emph{Be Decisive} leverages the layout implicitly encoded in the initial noise and refines it during denoising, avoiding conflicts with externally imposed layouts and improving prompt alignment while preserving model priors \cite{dahary_be_2025}. These approaches disentangle primarily via spatial decoupling but often require user-specified or learned layouts, which increases user effort and restricts spontaneous subject interaction. Our method reduces entanglement without explicit spatial annotations, requiring only the text prompt. 

\medtit{Training-time objectives for multi-subject fidelity.}
Some works alter training signals to strengthen subject–attribute binding. \emph{TokenCompose} introduces token-level supervision to improve consistency for prompts with multiple categories and attributes \cite{wang_tokencompose_2024}. Region-aware objectives decompose complex prompts into per-region descriptions and enforce alignment, reducing cross-entity leakage. Such methods typically assume curated supervision and substantial retraining. In contrast, our fine-tuning objective is lightweight: it adapts pre-trained models via Adjoint Matching and requires only text prompts, while our test-time controller operates with zero parameter updates.

%% file: sections/emprical_results.tex
\begin{table*}[t]
    \caption{Test-time control results for each heuristic. We report mean $\pm$ std over all prompts and seeds; the top three values per metric are highlighted (gold/silver/bronze). All methods use the same sampling and evaluation pipeline and are shown at their respective optimal $\lambda$.}
    \vspace{-1.5mm}
    \label{tab:metrics_otf}
    \centering
    \resizebox{\textwidth}{!}{
    \renewcommand{\arraystretch}{1.05}
    \setlength{\tabcolsep}{1.5mm}%
	\begin{tabular}{@{}c<{\enspace}@{}lcccccc|cc@{}}\toprule
    & \textbf{Heuristic} & \textbf{CLIP  I-T$\uparrow$} & \textbf{SigLIP-2 I-T$\uparrow$} & \textbf{BLIP T-T$\uparrow$} & \textbf{Qwen2 T-T$\uparrow$} & \textbf{PickScore I-T$\uparrow$} & \textbf{ImgRew I-T$\uparrow$} & \textbf{Composite$\uparrow$} \\
    \midrule
    \multirow{5}{*}{\rotatebox{90}{\textbf{SD 3.5}}}
    & Base & 0.3474\scriptsize$\pm$0.03 & 0.2309\scriptsize$\pm$0.05 & 0.5731\scriptsize$\pm$0.15 & \cellcolor{bronze!20}0.6402\scriptsize$\pm$0.08 & 22.6940\scriptsize$\pm$0.99 & 1.3175\scriptsize$\pm$0.68 & 0.0000\scriptsize$\pm$0.00 \\
    & Attend\&Excite & \cellcolor{silver!20}0.3484\scriptsize$\pm$0.03 & \cellcolor{silver!20}0.2326\scriptsize$\pm$0.05 & \cellcolor{silver!20}0.5752\scriptsize$\pm$0.15 & \cellcolor{silver!20}0.6404\scriptsize$\pm$0.08 & \cellcolor{bronze!20}22.6950\scriptsize$\pm$1.01 & \cellcolor{bronze!20}1.3545\scriptsize$\pm$0.66 & 3.1714\scriptsize$\pm$52.62 \\
    & CONFORM & 0.3481\scriptsize$\pm$0.03 & \cellcolor{bronze!20}0.2323\scriptsize$\pm$0.05 & \cellcolor{gold!20}\textbf{0.5773}\scriptsize$\pm$0.15 & \cellcolor{gold!20}\textbf{0.6421}\scriptsize$\pm$0.08 & \cellcolor{silver!20}22.7188\scriptsize$\pm$0.99 & \cellcolor{silver!20}1.3684\scriptsize$\pm$0.64 & \cellcolor{bronze!20}3.4336\scriptsize$\pm$44.89 \\
    & Divide\&Bind & \cellcolor{gold!20}\textbf{0.3489}\scriptsize$\pm$0.03 & 0.2316\scriptsize$\pm$0.05 & 0.5742\scriptsize$\pm$0.14 & 0.6399\scriptsize$\pm$0.08 & 22.6779\scriptsize$\pm$1.03 & 1.3493\scriptsize$\pm$0.67 & \cellcolor{silver!20}3.9373\scriptsize$\pm$77.81 \\
    & FOCUS(Ours) & \cellcolor{bronze!20}0.3483\scriptsize$\pm$0.03 & \cellcolor{gold!20}\textbf{0.2344}\scriptsize$\pm$0.04 & \cellcolor{bronze!20}0.5751\scriptsize$\pm$0.15 & 0.6385\scriptsize$\pm$0.08 & \cellcolor{gold!20}\textbf{22.7499}\scriptsize$\pm$1.02 & \cellcolor{gold!20}\textbf{1.4003}\scriptsize$\pm$0.62 & \cellcolor{gold!20}\textbf{4.2865}\scriptsize$\pm$85.86 \\
    \midrule
    \multirow{5}{*}{\rotatebox{90}{\textbf{FLUX.1}}}
    & Base & \cellcolor{silver!20}0.3449\scriptsize$\pm$0.03 & \cellcolor{silver!20}0.2271\scriptsize$\pm$0.05 & \cellcolor{silver!20}0.5739\scriptsize$\pm$0.15 & 0.6300\scriptsize$\pm$0.09 & \cellcolor{bronze!20}23.4234\scriptsize$\pm$1.03 & \cellcolor{gold!20}\textbf{1.2970}\scriptsize$\pm$0.66 & 0.0000\scriptsize$\pm$0.00 \\
    & Attend\&Excite & 0.3430\scriptsize$\pm$0.03 & 0.2242\scriptsize$\pm$0.05 & 0.5716\scriptsize$\pm$0.14 & 0.6304\scriptsize$\pm$0.09 & 23.2549\scriptsize$\pm$1.11 & 1.2494\scriptsize$\pm$0.70 & \cellcolor{silver!20}1.7595\scriptsize$\pm$67.56 \\
    & CONFORM & 0.3436\scriptsize$\pm$0.03 & 0.2252\scriptsize$\pm$0.05 & \cellcolor{bronze!20}0.5726\scriptsize$\pm$0.15 & \cellcolor{bronze!20}0.6321\scriptsize$\pm$0.09 & 23.3574\scriptsize$\pm$1.03 & 1.2461\scriptsize$\pm$0.70 & 1.5114\scriptsize$\pm$26.28 \\
    & Divide\&Bind & \cellcolor{gold!20}\textbf{0.3453}\scriptsize$\pm$0.03 & \cellcolor{gold!20}\textbf{0.2272}\scriptsize$\pm$0.05 & 0.5722\scriptsize$\pm$0.15 & \cellcolor{gold!20}\textbf{0.6330}\scriptsize$\pm$0.08 & \cellcolor{gold!20}\textbf{23.4395}\scriptsize$\pm$1.02 & \cellcolor{silver!20}1.2939\scriptsize$\pm$0.67 & \cellcolor{bronze!20}1.6352\scriptsize$\pm$44.00 \\
    & FOCUS(Ours) & \cellcolor{bronze!20}0.3446\scriptsize$\pm$0.03 & \cellcolor{bronze!20}0.2268\scriptsize$\pm$0.05 & \cellcolor{gold!20}\textbf{0.5741}\scriptsize$\pm$0.14 & \cellcolor{silver!20}0.6326\scriptsize$\pm$0.08 & \cellcolor{silver!20}23.4274\scriptsize$\pm$1.02 & \cellcolor{bronze!20}1.2913\scriptsize$\pm$0.67 & \cellcolor{gold!20}\textbf{1.9712}\scriptsize$\pm$31.77 \\
    \bottomrule
  \end{tabular}}
\end{table*}

\section{Experiments}
We evaluate our approach along two axes: (i) SOC-based test-time control vs.\ fine-tuning, and (ii) our probabilistic loss \textsc{FOCUS} vs.\ existing attention heuristics used as running costs $f$. We first describe setup, then report \emph{on-the-fly} and \emph{fine-tuning} results. Further details, ablations, and qualitative examples are provided in the Supplementary Material. All experiments ran on NVIDIA A100/H100 GPUs. While the test-time controller runs on commodity GPUs with as little as 12 GB VRAM, fine-tuning experiments fit within the VRAM of H100 GPUs. In terms of runtime, test-time steering requires backpropagation, increasing latency (+138\% for SD3, +90\% for FLUX); the fine-tuned models add zero inference-time overhead over the base model.

\medtit{Base Models.} 
We report results on two open-source flow-matching models: \emph{Stable Diffusion 3.5} (SD~3.5) \cite{esser_scaling_2024} and \emph{FLUX.1 [dev]} (FLUX.1) \cite{labs_flux1_2025}.

\medtit{Dataset.}
We create a 150-prompt corpus with 2–4 subjects per prompt using GPT-5. Half the prompts contain \emph{similar} subjects (\eg{,} ``a black bear and a brown bear[...]''); the rest contain \emph{dissimilar} subjects (\eg{,} ``a snowboard, a telescope, and a husky[...]''). For each prompt, we annotate subject token indices for both CLIP and T5 encoders to allow extraction cross-attention maps. Such per-subject annotations are typically absent from existing corpora.

\begin{table*}
    \caption{Fine-tuning results for each heuristic. We report mean $\pm$ std over all prompts and seeds; the top three values per metric are highlighted (gold/silver/bronze). All methods use the same sampling and evaluation pipeline and are shown at their respective optimal $\lambda$.}\vspace{-1.5mm}
    \label{tab:metrics_finetune}
    \centering
    \resizebox{\textwidth}{!}{
    \renewcommand{\arraystretch}{1.05}
    \setlength{\tabcolsep}{1.5mm}%
	\begin{tabular}{@{}c<{\enspace}@{}lcccccc|cc@{}}\toprule
    & \textbf{Heuristic} & \textbf{CLIP  I-T$\uparrow$} & \textbf{SigLIP-2 I-T$\uparrow$} & \textbf{BLIP T-T$\uparrow$} & \textbf{Qwen2 T-T$\uparrow$} & \textbf{PickScore I-T$\uparrow$} & \textbf{ImgRew I-T$\uparrow$} & \textbf{Composite$\uparrow$} \\
    \midrule
    \multirow{5}{*}{\rotatebox{90}{\textbf{SD 3.5}}}
    & Base & 0.3474\scriptsize$\pm$0.03 & \cellcolor{silver!20}0.2309\scriptsize$\pm$0.05 & 0.5731\scriptsize$\pm$0.15 & \cellcolor{silver!20}0.6402\scriptsize$\pm$0.08 & \cellcolor{silver!20}22.6940\scriptsize$\pm$0.99 & 1.3175\scriptsize$\pm$0.68 & 0.0000\scriptsize$\pm$0.00 \\
    & Attend\&Excite & 0.3469\scriptsize$\pm$0.03 & 0.2281\scriptsize$\pm$0.04 & \cellcolor{silver!20}0.5747\scriptsize$\pm$0.15 & \cellcolor{gold!20}\textbf{0.6425}\scriptsize$\pm$0.08 & \cellcolor{gold!20}\textbf{22.8429}\scriptsize$\pm$1.01 & \cellcolor{silver!20}1.4460\scriptsize$\pm$0.60 & \cellcolor{silver!20}5.7181\scriptsize$\pm$121.76 \\
    & CONFORM & \cellcolor{bronze!20}0.3478\scriptsize$\pm$0.03 & \cellcolor{bronze!20}0.2294\scriptsize$\pm$0.05 & 0.5646\scriptsize$\pm$0.15 & \cellcolor{bronze!20}0.6393\scriptsize$\pm$0.09 & 22.5962\scriptsize$\pm$0.99 & \cellcolor{bronze!20}1.3782\scriptsize$\pm$0.63 & \cellcolor{bronze!20}3.4583\scriptsize$\pm$105.28 \\
    & Divide\&Bind & \cellcolor{silver!20}0.3486\scriptsize$\pm$0.03 & 0.2266\scriptsize$\pm$0.05 & \cellcolor{gold!20}\textbf{0.5870}\scriptsize$\pm$0.14 & 0.6358\scriptsize$\pm$0.08 & 22.3401\scriptsize$\pm$0.99 & 1.3524\scriptsize$\pm$0.68 & 0.8006\scriptsize$\pm$69.71 \\
    & FOCUS(Ours) & \cellcolor{gold!20}\textbf{0.3495}\scriptsize$\pm$0.03 & \cellcolor{gold!20}\textbf{0.2331}\scriptsize$\pm$0.04 & \cellcolor{bronze!20}0.5744\scriptsize$\pm$0.15 & 0.6383\scriptsize$\pm$0.08 & \cellcolor{bronze!20}22.6445\scriptsize$\pm$0.97 & \cellcolor{gold!20}\textbf{1.4495}\scriptsize$\pm$0.58 & \cellcolor{gold!20}\textbf{5.9174}\scriptsize$\pm$119.94 \\
    \midrule
    \multirow{5}{*}{\rotatebox{90}{\textbf{FLUX.1}}}
    & Base & 0.3449\scriptsize$\pm$0.03 & 0.2271\scriptsize$\pm$0.05 & 0.5739\scriptsize$\pm$0.15 & 0.6300\scriptsize$\pm$0.09 & \cellcolor{gold!20}\textbf{23.4234}\scriptsize$\pm$1.03 & 1.2970\scriptsize$\pm$0.66 & 0.0000\scriptsize$\pm$0.00 \\
    & Attend\&Excite & \cellcolor{gold!20}\textbf{0.3468}\scriptsize$\pm$0.03 & \cellcolor{silver!20}0.2320\scriptsize$\pm$0.05 & \cellcolor{gold!20}\textbf{0.5876}\scriptsize$\pm$0.15 & \cellcolor{silver!20}0.6382\scriptsize$\pm$0.08 & \cellcolor{bronze!20}23.3333\scriptsize$\pm$1.01 & \cellcolor{silver!20}1.3806\scriptsize$\pm$0.62 & \cellcolor{silver!20}2.3477\scriptsize$\pm$79.28 \\
    & CONFORM & \cellcolor{bronze!20}0.3458\scriptsize$\pm$0.03 & \cellcolor{bronze!20}0.2305\scriptsize$\pm$0.04 & \cellcolor{silver!20}0.5800\scriptsize$\pm$0.15 & \cellcolor{bronze!20}0.6369\scriptsize$\pm$0.08 & \cellcolor{silver!20}23.3724\scriptsize$\pm$1.00 & \cellcolor{bronze!20}1.3631\scriptsize$\pm$0.63 & \cellcolor{bronze!20}1.9591\scriptsize$\pm$82.88 \\
    & Divide\&Bind & 0.3445\scriptsize$\pm$0.03 & 0.2296\scriptsize$\pm$0.05 & 0.5705\scriptsize$\pm$0.15 & 0.6246\scriptsize$\pm$0.09 & 23.1909\scriptsize$\pm$1.06 & 1.2269\scriptsize$\pm$0.70 & 0.2002\scriptsize$\pm$47.34 \\
    & FOCUS(Ours) & \cellcolor{silver!20}0.3468\scriptsize$\pm$0.03 & \cellcolor{gold!20}\textbf{0.2328}\scriptsize$\pm$0.05 & \cellcolor{bronze!20}0.5780\scriptsize$\pm$0.15 & \cellcolor{gold!20}\textbf{0.6386}\scriptsize$\pm$0.08 & 23.3278\scriptsize$\pm$1.01 & \cellcolor{gold!20}\textbf{1.3899}\scriptsize$\pm$0.61 & \cellcolor{gold!20}\textbf{2.5881}\scriptsize$\pm$78.83 \\
    \bottomrule
  \end{tabular}}
\end{table*}

\medtit{Metrics.} 
Following \citet{yu_attention_2025}, we report two alignment groups. For image–text (I–T) alignment we compute CLIP~\cite{radford_clip} and SigLIP-2~\cite{tschannen_siglip_2025} cosine similarities. For caption-based text–text (T–T) faithfulness, we caption each image with BLIP~\cite{li_blip_2022} and Qwen2-VL~\cite{wang_qwen2-vl_2024} and measure similarity to the prompt. To better reflect human preference and penalize artifacts, we additionally report PickScore \cite{kirstain_pick--pic_2023} and ImageReward \cite{xu_imagereward_2023}. 

For model selection, we use a \emph{composite} score obtained by macro-averaging baseline-relative gains across all metrics. We generate five images per prompt and hyperparameter setting (distinct seeds), fixing sampler, steps, guidance, and resolution. Metric variants and computation details are given in the Supplementary Material.

\medtit{Baselines and heuristics.}
To study the portability of attention-space heuristics under a unified control framework, we consider \emph{Attend\&Excite} \cite{chefer_attend-and-excite_2023}, \emph{CONFORM} \cite{meral_conform_2023}, and \emph{Divide\&Bind} \cite{li_divide_2024}, and our probabilistic loss \textsc{FOCUS}. Each heuristic is instantiated as a running cost $f_h$ in our SOC objective and scaled by a single strength parameter $\lambda$.

\medtit{Control strength.}
Each running-cost heuristic $f_h$ is scaled by a single strength parameter $\lambda$, \ie{,} $f_{\lambda}(x,t)=\lambda\,f_h(x,t)$, which controls how strongly we steer away from the base sampler. For every heuristic and model, we sweep a fixed grid of $\lambda$ values and report results at the \emph{per-heuristic} optimum selected by the composite score.

\medtit{Human study.} 
Because automated metrics under-detect subtle entanglement \cite{dahary_be_2025}, we run a prompt-conditioned pairwise preference study with 50 participants. Each trial shows a prompt and two images; participants select the better match, yielding \(2{,}000\) comparisons. We report win rates and Elo ratings; protocol and per-prompt results appear in the Supplementary Material.

\subsection{On-the-fly disentanglement (test-time control)}
\begin{table}
    \robustify\bfseries
    \centering
    \caption{Human preference study for test-time control and fine-tuned models. Report pairwise win rate and Elo rating.}
    \label{tab:tab_human}
    \vspace{-2mm}
    \resizebox{0.98\linewidth}{!}{
        \begin{tabular}{@{}c<{\enspace}@{}lccccc@{}}
        \toprule
        & \multirow{2}{*}{\textbf{Heuristic}} & \multicolumn{2}{c}{\textbf{SD 3.5}} & \multicolumn{2}{c}{\textbf{FLUX.1}}\\
        \cmidrule(r){3-4}\cmidrule(l){5-6}
        & & {Win[\%]} & {Elo$\uparrow$} & {Win[\%]} & {Elo$\uparrow$}\\
        \midrule
        \multirow{5}{*}{\rotatebox{90}{\textbf{Test-Time}}}
        & Base & {45\%} & {\cellcolor{bronze!20}1517} & {46\%} & {1464} \\
        & Attend\&Excite & {\cellcolor{silver!20}53\%} & {1500} & {49\%} & {\cellcolor{silver!20}1526} \\
        & CONFORM & {42\%} & {1373} & {\cellcolor{silver!20}50\%} & {\cellcolor{bronze!20}1498} \\
        & Divide\&Bind & {\cellcolor{bronze!20}50\%} & {\cellcolor{gold!20}1562} & {\cellcolor{silver!20}50\%} & {1450} \\
        & \textsc{FOCUS} (Ours) & {\cellcolor{gold!20}\textbf{58\%}} & {\cellcolor{silver!20}1548} & {\cellcolor{gold!20}\textbf{54\%}} & {\cellcolor{gold!20}\textbf{1562}} \\
        \midrule
        \multirow{5}{*}{\rotatebox{90}{\textbf{Fine-tuning}}}
        & Base & {39\%} & {1355} & {\cellcolor{silver!20}51\%} & {1462} \\
        & Attend\&Excite & {\cellcolor{silver!20}56\%} & {\cellcolor{silver!20}1584} & {\cellcolor{bronze!20}50\%} & {\cellcolor{bronze!20}1476} \\
        & CONFORM & {\cellcolor{bronze!20}49\%} & {\cellcolor{bronze!20}1520} & {50\%} & {\cellcolor{gold!20}\textbf{1620}} \\
        & Divide\&Bind & {48\%} & {1436} & {43\%} & {1442} \\
        & \textsc{FOCUS} (Ours) & {\cellcolor{gold!20}\textbf{57\%}} & {\cellcolor{gold!20}\textbf{1605}} & {\cellcolor{gold!20}\textbf{54\%}} & {\cellcolor{silver!20}1500} \\
        \bottomrule
        \end{tabular}
    }
\end{table}
We sweep ten $\lambda$ values per heuristic and select the best via the composite score. \Cref{tab:metrics_otf} reports per-heuristic, per-model metrics at the optimal $\lambda$; representative samples are shown in \cref{fig:otf}, and human preferences in \cref{tab:tab_human}.

Across SD~3.5 and FLUX.1, all SOC-controlled heuristics improve over the base sampler on the composite score, indicating that the SOC formulation offers a principled way to port legacy attention heuristics to modern flow-matching models. Qualitatively, subjects are more reliably present and more clearly separated.

\textsc{FOCUS} is competitive on individual automatic metrics and achieves one of the highest composite scores on both backbones. In the human study, \textsc{FOCUS} obtains the best win rates and Elo ratings for both SD~3.5 and FLUX.1, suggesting that a probabilistic attention loss is a particularly strong choice for test-time SOC control.

FOCUS also improves \emph{native} latent-update guidance: replacing CONFORM’s original objective in its SD1.5 pipeline increases average alignment metrics by $+2.76\%$.

\subsection{Fine-tuning for disentanglement}
We next evaluate \emph{fine-tuned} SOC controllers learned via Adjoint Matching. We insert LoRA layers \cite{hu_lora_2021} with rank $r{=}4$ into self-attention blocks, freeze all other parameters, and update $<0.1\%$ of weights. We sweep $\lambda$ and a small set of optimization and data hyperparameters. \Cref{tab:metrics_finetune} reports quantitative results, \cref{fig:fine} shows samples, and \cref{tab:tab_human} summarizes human preferences.

Fine-tuned models consistently outperform their test-time counterparts on the composite score for both SD~3.5 and FLUX.1, while preserving base style and diversity. This aligns with our analysis: fine-tuning integrates the lean adjoint over full trajectories, whereas test-time control relies on a local single-pass approximation.

\medtit{Data efficiency.}
To assess generalization, we fine-tune \textsc{FOCUS} on three subsets of our corpus: 1 prompt, 15 prompts, and all 150 prompts. \Cref{tab:fine_data} shows composite scores. A single training prompt already yields strong gains on SD~3.5 and competitive gains on FLUX.1, indicating that the SOC controller learns a broadly useful disentangling direction from very limited supervision. Larger datasets slightly reduce peak composite improvements but improve robustness across prompts (see Supplementary Material).

\begin{table}
    \centering
    \caption{Composite scores for \textsc{FOCUS} fine-tuning across subsets.}
    \vspace{-2mm}
    \label{tab:fine_data}
    \resizebox{0.9\linewidth}{!}{
        \begin{tabular}{@{}lccc@{}}
        \toprule
        \textbf{Model} & 1 prompt & 15 prompts & 150 prompts\\
        \midrule
        \textbf{SD~3.5} & 5.917\scriptsize$\pm$119.94 & 3.191\scriptsize$\pm$99.79 & 1.682\scriptsize$\pm$53.85\\
        \textbf{FLUX.1} & 2.588\scriptsize$\pm$78.83 & 2.457\scriptsize$\pm$86.64 & 1.810\scriptsize$\pm$90.13\\
        \bottomrule
        \end{tabular}
    }
\end{table}

\medtit{Human preferences.}
Human judgments mirror the quantitative trends. As shown in \cref{tab:tab_human}, \textsc{FOCUS} attains the highest win rates and Elo ratings among all heuristics for SD~3.5, and the highest win rate (with competitive Elo) for FLUX.1. Together with the automatic metrics, this indicates that \textsc{FOCUS} is an effective running cost for both test-time and fine-tuned SOC controllers, and that the SOC framework provides a unified, architecture-agnostic way to improve multi-subject disentanglement.

\subsection{Comparison to Self-Cross Guidance}
For completeness, we also evaluate Self-Cross Guidance (SC Guidance) \cite{qiu2025self}, which, unlike our other baselines, jointly optimizes cross- and self-attention over image tokens. This makes it less directly comparable to methods that intervene only in cross-attention.

Using the same evaluation protocol, SC Guidance attains the highest test-time composite scores ($4.89$ / $1.97$ on SD~3.5 / FLUX.1) but degrades under fine-tuning ($3.75$ / $-0.58$), in the latter case even underperforming the base model; see \cref{fig:otf,fig:fine} for examples. Qualitatively, we observe pronounced segmentation artifacts, suggesting that aggressively shaping self-attention can harm image quality. Because of these failure modes and its different intervention space, we exclude SC Guidance from our human study and focus user evaluations on the more comparable cross-attention–based heuristics.

%% file: sections/discussion.tex

\section{Discussion}
We introduced a control-theoretic framework for improving multi-subject fidelity in flow-matching T2I models. A single SOC objective yields two instantiations: a single-pass test-time controller and a lightweight fine-tuned controller. The framework treats prior attention heuristics as interchangeable running costs, and our probabilistic loss \textsc{FOCUS} provides the most consistent gains across backbones and evaluation settings. The two instantiations trade compute for convenience. Test-time control steers a frozen model given subject tokens, but increases inference time by roughly $2\times$. Fine-tuning requires subject tokens only during training and preserves baseline inference speed, while generalizing well beyond its training data.


%% file: sections/appendix.tex
\section{FOCUS Heuristic}\label{app:focus}
In this section we detail \textsc{FOCUS}, our probabilistic attention heuristic used as a running cost for disentanglement. We emphasize three practical design choices: (i) encoding spatial proximity before measuring divergence, (ii) aggregating attention maps prior to scoring, and (iii) omitting an explicit collapse regularizer.

\paragraph{Spatially aware divergence.}
We promote separation of subjects by maximizing a Jensen--Shannon divergence (JSD) defined over attention distributions. A na\"ive computation on flattened maps discards locality, allowing distant activations to interact as if adjacent. To preserve spatial structure, we (i) reshape token-embedding maps to the target aspect ratio (height adn width of the resulting image), (ii) apply a simple 2D Gaussian smoothing (symmetric kernel of size 3), and only then (iii) flatten for scoring. This encodes proximity into the flattened attention tensor and mitigates grid-like artifacts during optimization.

\paragraph{Block selection and aggregation.}
Modern T2I backbones follow Diffusion Transformer designs \cite{peebles_scalable_2023}. Rather than computing scores \emph{per block} and averaging their scores which can result in conflicting update directions, we first aggregate attention and then score. Concretely, we average cross-attention maps over all blocks that process text and image tokens \emph{separately}, producing a single map per token and a consistent optimization direction. Blocks that jointly process text and image tokens are excluded from this aggregation for compatibility.

\paragraph{No explicit collapse regularizer.}
We experimented with an entropy-based regularizer aimed at discouraging overly concentrated (collapsed) attention. Let $H(\vp) = -\sum_i \vp_i \log \vp_i$ denote the Shannon Entropy and $\widehat{H}(\vp)=H(\vp)/\log d \in[0,1]$ its normalized version, where $d$ is the number of spatial locations. For each subject we form its mixture distribution $\vm_\vs$ and added
\begin{align}
    \gamma_\text{reg} \cdot \frac{1}{|S|}\sum_{\vs \in S}\left( 1 - \widehat{H}(\vm_s)\right),
\end{align}
scaling by $\gamma_\text{reg} > 0$ to control its effect. In our experiments, small $\gamma_\text{reg}$ made the term largely inactive, while larger $\gamma_\text{reg}$ pushed mass away from the subject rather than stabilizing it, see \cref{fig:reg} for an example. We therefore omit this term in \textsc{FOCUS} and rely on the probabilistic objective described above.

\begin{figure}[htb]
    \centering
    \begin{subfigure}[t]{0.19\linewidth}
        \centering
        \captionsetup{labelformat=empty}
        \caption{Base}
        \includegraphics[width=\linewidth]{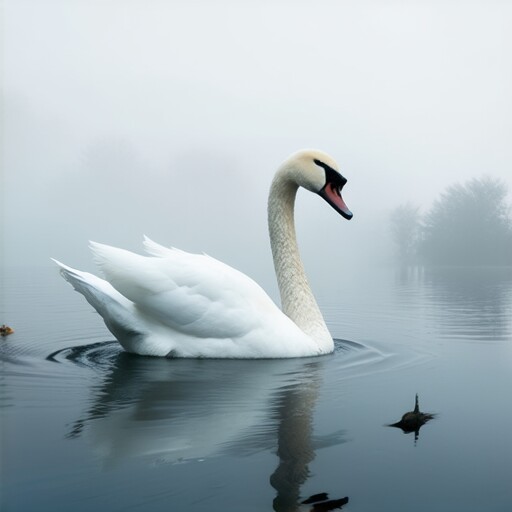}
    \end{subfigure}\hfill
    \begin{subfigure}[t]{0.19\linewidth}
        \centering
        \captionsetup{labelformat=empty}
        \caption{$\gamma_\text{reg}=0$}
        \includegraphics[width=\linewidth]{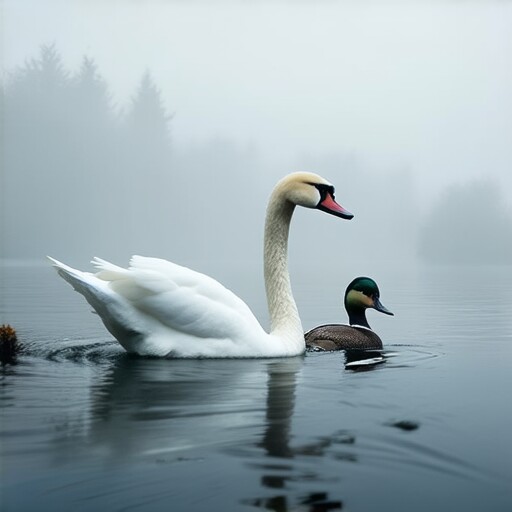}
    \end{subfigure}\hfill
    \begin{subfigure}[t]{0.19\linewidth}
        \centering
        \captionsetup{labelformat=empty}
        \caption{$\gamma_\text{reg}=0.01$}
        \includegraphics[width=\linewidth]{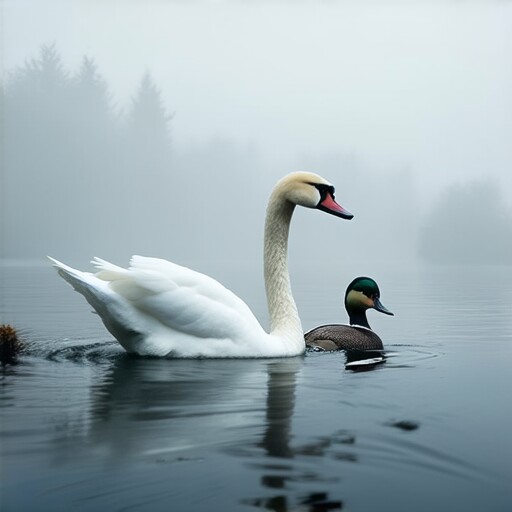}
    \end{subfigure}
    \begin{subfigure}[t]{0.19\linewidth}
        \centering
        \captionsetup{labelformat=empty}
        \caption{$\gamma_\text{reg}=1$}
        \includegraphics[width=\linewidth]{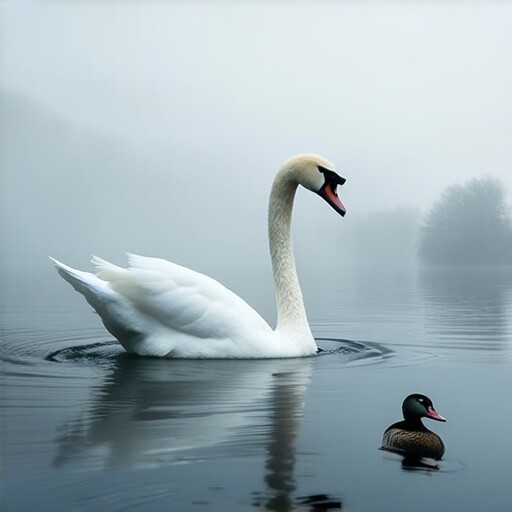}
    \end{subfigure}
    \begin{subfigure}[t]{0.19\linewidth}
        \centering
        \captionsetup{labelformat=empty}
        \caption{$\gamma_\text{reg}=10$}
        \includegraphics[width=\linewidth]{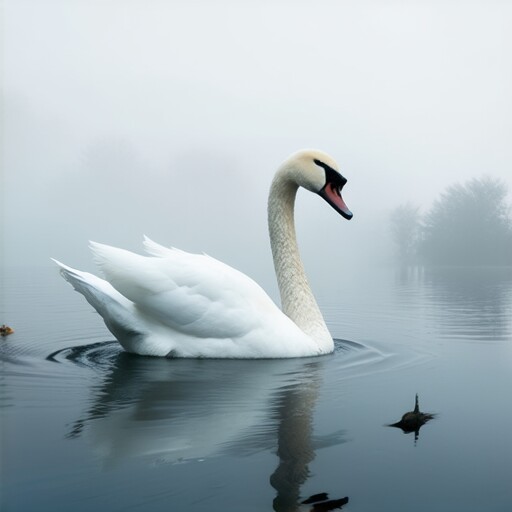}
    \end{subfigure}
    \caption{Ablation of regularizer strength $\gamma_{\mathrm{reg}}$ for the test-time controller on Stable Diffusion 3.5.}
    \label{fig:reg}
\end{figure}

\begin{lemma}[Upper Bound of Jensen--Shannon Divergence]\label{lem:jsd}
    Let $P = \{\vp^{(1)}, \dots, \vp^{(n)}\} \subset\Delta^{d-1}$ be a set of probability distributions. Then, $D_\mathrm{JS}(P)$ is upper bounded by $\log n$.   
\end{lemma}
\begin{proof}
    Define $P$ as in \cref{lem:jsd}, then the JSD is defined as follows:
    \begin{align*}
        D_\mathrm{JS}(P) = \frac{1}{n}\sum_{k=1}^{n} D_\mathrm{KL}\left(\vp^{(k)} \;\|\; \vm\right), \quad \vm = \frac{1}{n}\sum_{k=1}^{n} \vp^{(k)}.
    \end{align*}
    We can upper bound each $D_\mathrm{KL}$-term as follows:
    \begin{align*}
        D_\mathrm{KL} (\vp^{(k)} \;\|\; \vm) &= \sum_{i=1}^d p^{(k)}_i \log \frac{p^{(k)}_i}{m_i}\\
        &= \sum_{i=1}^d {p}^{(k)}_i \log \frac{{p}^{(k)}_i}{\frac{1}{n}\sum_{\ell=1}^n p^{(\ell)}_i}\\
        &= \sum_{i=1}^d {p}^{(k)}_i \log \left( n \cdot \frac{p^{(k)}_i}{\sum_{\ell=1}^n p^{(\ell)}_i} \right)\\
        &\leq \sum_{i=1}^d {p}^{(k)}_i \log n\\
        &= \log n.
    \end{align*}
    Plugging this bound back into the definition of the JSD, yields the desired results:
    \begin{align*}
         \frac{1}{n}\sum_{k=1}^{n} D_\mathrm{KL}\left(\vp^{(k)} \;\|\; \vm\right) \leq \frac{1}{n}\sum_{k=1}^{n} \log n = \log n
    \end{align*}
\end{proof}

\paragraph{Normalization.}
Because \(D_{\mathrm{JS}}(P)\in[0,\log n]\), we use the normalized score
\(\widehat{D}_{\mathrm{JS}}(P)=D_{\mathrm{JS}}(P)/\log n \in[0,1]\),
which makes values comparable across different set sizes.

\section{Denoising Diffusion as Flow Matching}\label{app:denoising}
This section makes precise how classical denoising diffusion (score-based) models arise as a special case of the flow-matching (FM) framework. We first derive the continuous-time SDE limit of the variance-preserving (VP) family \cite{sohl-dickstein_deep_2015,ho_denoising_2020,song_score-based_2021}; after we express reverse-time generation; and finally show an explicit parameterization that uses a diffusion model as an FM velocity field. Analogous statements hold for VE and EDM variants \cite{nichol_improved_2021,karras_elucidating_2022,karras_analyzing_2024}.

\subsection{Variance-Preserving chain to SDE}
Let $X_0 \sim p_{\mathrm{data}}$. The standard $K$-step VP forward noising chain is
\begin{align}
    X_k = \sqrt{\alpha_k} X_{k-1} + \sqrt{1-\alpha_k}\epsilon_k, \quad \epsilon_k \sim \mathcal N(0,\mI), \quad k=1,\dots,K,
\end{align}
where $\alpha_k := 1-\beta_k \in (0,1)$ with $\beta_k\in(0,1)$ typically increasing over $k$ \cite{ho_denoising_2020}. This yields the closed-form marginal
\begin{align}
X_k \mid X_0 \sim \mathcal N \big(\sqrt{\bar\alpha_k}\,X_0,\ (1-\bar\alpha_k)\,\mI\big),
\qquad \bar\alpha_k := \prod_{i=1}^k \alpha_i. \label{eq:ddpm-marginal}
\end{align}
For sufficiently large $K$, $X_K$ is approximately standard normal \cite{ho_denoising_2020}.

We lift this formulation to continuous time by defining a uniform grid $\tau_k := k/K$ on $[0,1]$, so every increment is $\Delta \tau = 1/K$. Define a piecewise-constant rate $\beta(\tau)$ via $\beta(\tau) := \beta_k/\Delta \tau$ for $\tau \in [\tau_{k-1}, \tau_k)$. Then by using the first-order Taylor approximation of $\sqrt{1+x}$, we can rewrite $\sqrt{\alpha_k} = \sqrt{1 - \beta_k} \approx 1 - \frac{1}{2}\beta_k + \mathcal{O}(\beta_k^2)$, and obtain
\begin{align}
    \Delta X_k &:= X_k - X_{k-1}\\
    &= -\frac{1}{2}\beta_k X_{k-1} + \sqrt{\beta_k} \epsilon_k + \mathcal{O}(\beta_k^2)\\
    &= \left( -\frac{1}{2}\beta(\tau_{k-1}) X_{k-1} \right) \Delta \tau + \sqrt{\beta(\tau_{k-1})} \sqrt{\Delta \tau}\epsilon_k + \mathcal{O}\left((\Delta \tau)^\frac{3}{2}\right).
\end{align}
This is the Euler--Maruyama discretizations of the forward/diffusion VP-SDE:
\begin{align}
    d X_\tau = -\frac{1}{2}\beta(\tau) X_\tau d\tau + \sqrt{\beta(\tau)} dB_\tau, \quad \tau \in [0,1],
\end{align}
and the discrete chain converges to this SDE as $K\rightarrow \infty$. Moreover, the SDE has Gaussian marginals
\begin{align}
    X_\tau \mid X_0 \sim \mathcal{N}\left(\sqrt{\bar{\alpha}}(\tau)X_0, \left(1 - \bar{\alpha}(\tau)\right)\mI\right), \quad\text{with}\quad \bar{\alpha}(\tau) := \exp\left( -\int_0^\tau \beta(u)du\right),
\end{align}
which matches \cref{eq:ddpm-marginal} at the grid points if we choose $\bar{\alpha}(\tau_k)=\bar{\alpha}_k$ \cite{song_score-based_2021}.

\subsection{Reverse-time dynamics}
We now reverse time to generate from noise to data. Let $\bar{\tau} = 1 - \tau$ denote the \emph{generative time}. By classical time reversal diffusion \cite{anderson_reverse-time_1982} the reverse-time process satisfies
\begin{align}
    d X_{\tau} = \left(-\frac{1}{2}\beta{\tau} X_{\tau} - \beta(\tau) \nabla_X \log p_\tau(X_{\tau})\right) d\bar{\tau} + \sqrt{\beta(\tau)} d \bar{B}_\tau, \quad\text{with}\quad d\bar{\tau} = - d\tau,
\end{align}
where $p_\tau$ are the forward-time marginals and $\nabla_X \log p_\tau$ is the score \cite{song_score-based_2021}.

In practice, most diffusion architectures parameterize the model via \emph{noise prediction} $\epsilon_\theta$ \cite{ho_denoising_2020,karras_elucidating_2022,karras_analyzing_2024}, which is related to the score by:
\begin{align}
    \nabla_X \log p_\tau(x) = - \frac{\epsilon_\theta(x, \tau)}{\sqrt{1 - \bar{\alpha}(\tau)}}.
\end{align}

\subsection{Time change to Flow Matching}
To embed VP diffusion into FM, we reparameterize time so that FM runs from noise to data, setting $t := 1-\tau$, which yields the following FM schedules:
\begin{align}
    \alpha^{\mathrm{FM}}_t := \sqrt{\bar\alpha(1-t)}, \quad\text{and}\quad
    \beta^{\mathrm{FM}}_t := \sqrt{1-\bar\alpha(1-t)}.
\end{align}

\subsection{Score relations}
For linear Gaussian reference paths, the score $s(x,t):=\nabla_X\log p_t(x)$ and the FM vector field $v_\theta(x,t)$ are linked by a schedule-dependent affine map \cite{lipman_flow_2023,albergo_stochastic_2023,liu_flow_2022}:
\begin{align}
    s(x,t) = \frac{1}{\eta_t}\Big(v_\theta(x,t) - \kappa_t\,x\Big),
    \qquad
    \kappa_t := \frac{\dot\alpha^{\mathrm{FM}}_t}{\alpha^{\mathrm{FM}}_t},\quad
    \eta_t := \beta^{\mathrm{FM}}_t\left(\frac{\dot\alpha^{\mathrm{FM}}_t}{\alpha^{\mathrm{FM}}_t} \beta^{\mathrm{FM}}_t - \dot\beta^{\mathrm{FM}}_t\right).
\end{align}

Combining the noise–score relation with the time change $\tau=1-t$ gives:
\begin{align}
    s(x,t) = \nabla_X \log p_t(x) = -\frac{\epsilon_\theta(x, 1-t)}{\beta^{\mathrm{FM}}_t},
\end{align}
since $\beta^{\mathrm{FM}}_t=\sqrt{1-\bar\alpha(1-t)}$. Substituting this into the score–velocity map yields the corresponding FM \emph{velocity prediction} induced by an $\epsilon$-parameterized diffusion model:
\begin{align}
    v_\theta(x,t) = \kappa_t x - \eta_t \frac{\epsilon_\theta(x,1-t)}{\beta^{\mathrm{FM}}_t}.
\end{align}
This identity lets an $\epsilon$-trained diffusion model be used directly as an FM velocity field for the VP-induced schedules above; plugging $v_\theta$ into the FM SDE recovers the reverse-time VP sampler (and setting $\sigma\equiv 0$ recovers the probability-flow/DDIM ODE) under the change of variables $t=1-\tau$.

\section{Evaluation Setup}\label{app:reproducibility}
We detail our evaluation pipeline: sampling parameters, the composite score over metrics, test-time control and fine-tuning hyperparameters, and an auxiliary open-vocabulary detection metric.

\subsection{Evaluation Sampling Parameters}\label{app:sampling_parameters}
For Stable Diffusion~3.5\footnote{\href{https://huggingface.co/stabilityai/stable-diffusion-3.5-medium}{https://huggingface.co/stabilityai/stable-diffusion-3.5-medium}} and FLUX.1 [dev]\footnote{\href{https://huggingface.co/black-forest-labs/FLUX.1-dev}{https://huggingface.co/black-forest-labs/FLUX.1-dev}}, we follow the official sampling recommendations. Unless stated otherwise, we use the deterministic Euler scheduler with $28$ inference steps for both models and generate images at $512\times512$ resolution. The classifier-free guidance scale is set to $4.5$ for SD3.5 and $3.5$ for FLUX. To ensure consistent extraction of cross-attention maps, we cap the maximum tokenized sequence length at $77$ for SD3.5 and $256$ for FLUX, and we verify that all prompts in our dataset fall within these limits. Models are loaded and all computations are performed in \texttt{bfloat16} to reduce memory usage.

\subsection{Evaluation Metrics}\label{app:metric}
To summarize each hyperparameter setting with a single scalar, we macro-average the \emph{relative improvement} over the base model across prompts, seeds, and metrics.

Let $X_{p,s}$ be the image produced by the current setting for prompt $p\in P$ and seed $s\in S$, and let $\hat{X}_{p,s}$ be the corresponding image from the base model. Let $\mathcal M$ denote the set of evaluation metrics. Since in our settings all metrics are increasing, the composite score for a hyperparameter setting is the macro-average
\begin{align}
    \frac{1}{|S|}\sum_{s \in S} \frac{1}{|P|}\sum_{p \in P} \frac{1}{|M|}\sum_{m \in M} \frac{m(X_{p, s}) - m(\hat{X}_{p, s})}{m(\hat{X}_{p, s})},
\end{align}
such that a value larger than $0$ indicates an average improvement over the base model, while values smaller than $0$ indicate degradation.

We conducted $160$ test-time control runs and $236$ fine-tuning runs, spanning multiple heuristics and hyperparameters. Detailed per-run results are in \Cref{sec:detailed_results}: test-time control for SD~3.5 in \Cref{tab:metrics_otf_sd3}, for FLUX.1~[dev] in \Cref{tab:metrics_otf_flux}, and for Stable Diffusion~1.5 in \Cref{tab:metrics_otf_sd1}; fine-tuning results for SD~3.5 in \Cref{tab:metrics_fine_sd3_1,tab:metrics_fine_sd3_2,tab:metrics_fine_sd3_3} and for FLUX.1~[dev] in \Cref{tab:metrics_fine_flux_1,tab:metrics_fine_flux_2,tab:metrics_fine_flux_3}.

\subsection{Test-Time Control: Deterministic Sampling (ODE)}
Many off-the-shelf T2I models are optimized for deterministic ODE sampling (\ie{,} $\sigma \equiv 0$). In this setting we view steering as an optimal control problem over the sampler ODE:
\begin{align}
    \min_{u} \;\; \E\left[\int_0^1 \frac{1}{2}\|u(X_t,t)\|_2^2 + f(X_t,t)\,dt \right]
    \quad \text{s.t.} \quad
    \dot X_t = v_{\text{base}}(X_t,t) + u(X_t,t), \;\; X_0 \sim \pi_0,
    \label{eq:ode_soc}
\end{align}
where $v_{\text{base}}$ is the base FM/diffusion velocity parameterization and $f$ is a disentanglement running cost (\eg{,} \textsc{FOCUS}).

Define the Hamiltonian
\begin{align}
    \mathcal H(x,u,a,t) \;=\; \frac{1}{2}\|u\|_2^2 + f(x,t) + a^\top\!\big(v_{\text{base}}(x,t)+u\big),
\end{align}
with co-state (adjoint) $a(t)\in\R^d$. Since $\mathcal H$ is strictly convex in $u$, the first-order condition gives
\begin{align}
    \nabla_u \mathcal H = u + a = 0
    \quad \Rightarrow \quad
    u_t^\star = -a(t).
    \label{eq:ode_u_star}
\end{align}
Computing $u_t^\star$ exactly would require integrating \cref{eq:adjoint} backward from $t=1$, which depends on future states and is incompatible with a single forward sampling pass. We therefore use a local approximation analogous to the SDE case: we evaluate gradients at the current state and drop the Jacobian term,
$\nabla_x v_{\text{base}}(X_t,t)^\top a(t)\approx 0$. This yields
\begin{align}
    \dot a(t) \approx -\nabla_x f(X_t,t),
    \qquad a(1)=0,
\end{align}
so that
\begin{align}
    a(t)
    \approx \int_t^1 \nabla_x f(X_t,\tau)\,d\tau
    \approx (1-t)\,\nabla_x f(X_t,t),
    \label{eq:ode_a_approx}
\end{align}
where the last step is a left-Riemann approximation (treating the future state locally constant). Substituting \cref{eq:ode_a_approx} into \cref{eq:ode_u_star} gives the instantaneous control
\begin{align}
    u_t^\star \approx -(1-t)\,\nabla_x f(X_t,t).
\end{align}
Therefore the controlled ODE uses the guided velocity
\begin{align}
    v_t^\star
    = v_{\text{base}}(X_t,t) + u_t^\star
    \approx v_{\text{base}}(X_t,t) - (1-t)\,\nabla_x f(X_t,t).
    \label{eq:ode_guided_velocity}
\end{align}
With an Euler integrator of step size $h$, this corresponds to
\begin{align}
    X_{t+h}
    = X_t + h\,v_t^\star
    \approx X_t + h\,v_{\text{base}}(X_t,t) - h(1-t)\,\nabla_x f(X_t,t),
\end{align}
which matches the standard latent-gradient guidance template with an explicit time schedule $\eta(t)=1-t$ (up to the global scaling of $f$ by $\lambda$). Dropping the Jacobian term $\nabla_x v_{\text{base}}^\top a$ is a common online-control approximation; we empirically find it yields a stable single-pass controller while keeping overhead low \cite{havens_adjoint_2025}.

\subsection{Test-Time Control: Hyperparameters}
In the deterministic (ODE) variant, the single-pass update does not inherit the time–weighting $\frac{1}{2}\sigma^2_\text{mem}(t)$ that appears in the SDE case. Since $\sigma_{\mathrm{mem}}(t)$ is large at early times and decays rapidly as $t \to 1$, we reintroduce this desirable early–strong / late–weak behavior in the ODE setting by reweighting the running cost:
\begin{align}
    f(X_t, t) = \lambda \cdot \sigma^2_\text{mem}(t) \cdot \mathrm{Heuristic}(X_t),\label{eq:impl_update}
\end{align}
where $\lambda > 0$ is the earlier introduced hyperparameter to account for different heuristic magnitudes. Throughout our test-time control experiments, we use this time-weighted running cost variant and sweep over $\lambda \in \{0.1, 0.5, 1, 2, 3, 4, 8, 12, 16, 32\}$. Values below $0.1$ have negligible effect across heuristics, while values above $32$ tend to produce artifacts (over-sharpening, texture noise) or occasional numerical instabilities (NaNs). See \cref{fig:lambda_example} for qualitative trends.

\begin{figure}[H]
    \centering
    {\setlength{\tabcolsep}{.3mm}%
    \resizebox{\linewidth}{!}{
    \begin{tabular}{*{10}{c}} 
        Base & $\lambda = 0.1$ & $\lambda = 0.5$ & $\lambda = 1$ &$\lambda = 2$ & $\lambda = 3$ & $\lambda = 4$ & $\lambda = 8$ & $\lambda = 16$ & $\lambda = 32$ \\
        \includegraphics[width=0.1\linewidth]{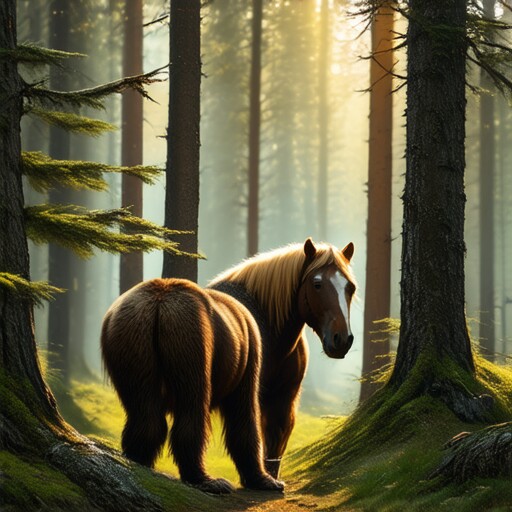} &
        \includegraphics[width=0.1\linewidth]{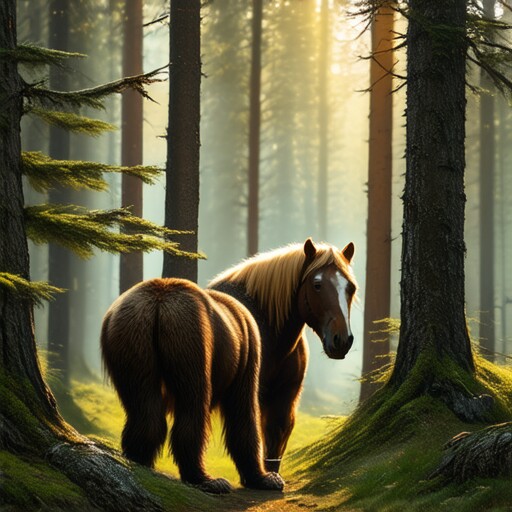} &
        \includegraphics[width=0.1\linewidth]{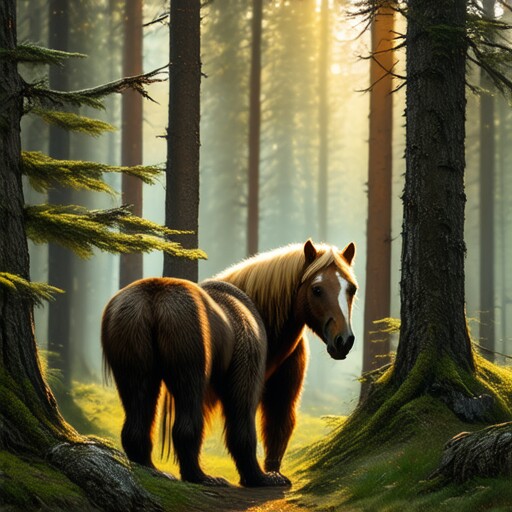} &
        \includegraphics[width=0.1\linewidth]{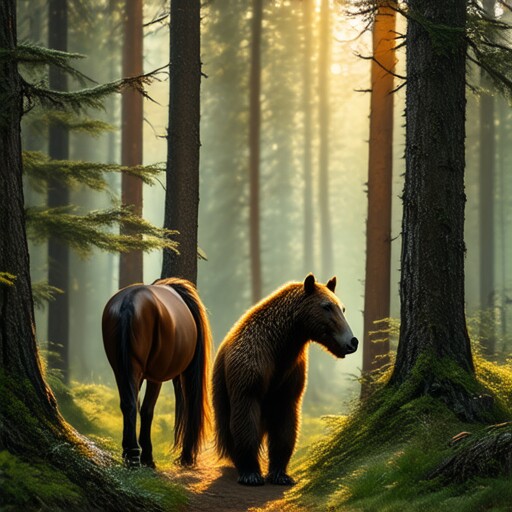} &
        \includegraphics[width=0.1\linewidth]{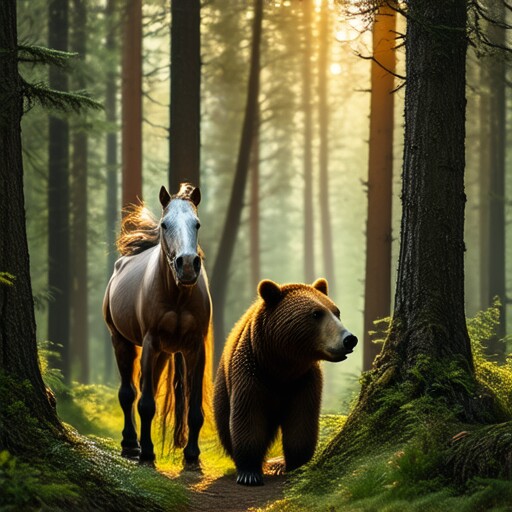} &
        \includegraphics[width=0.1\linewidth]{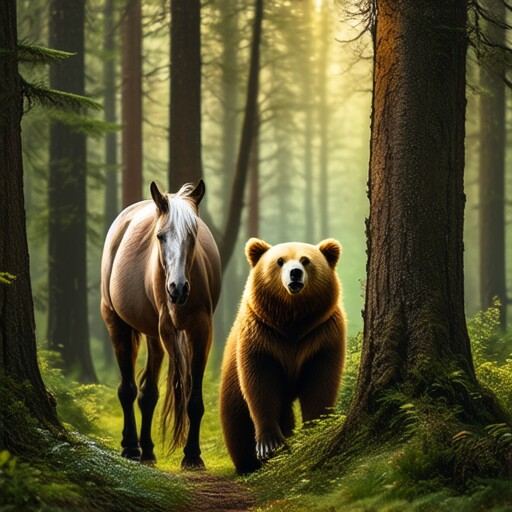} &
        \includegraphics[width=0.1\linewidth]{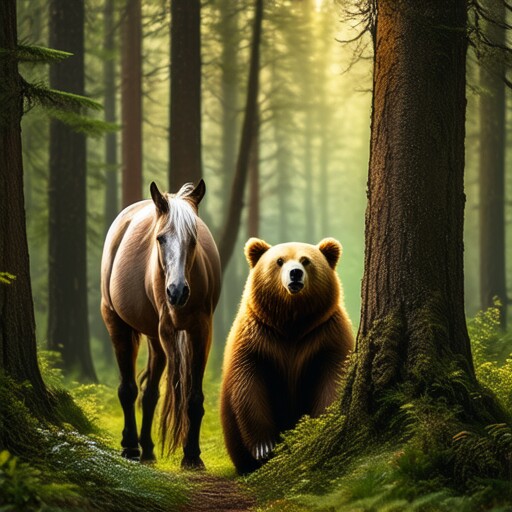} &
        \includegraphics[width=0.1\linewidth]{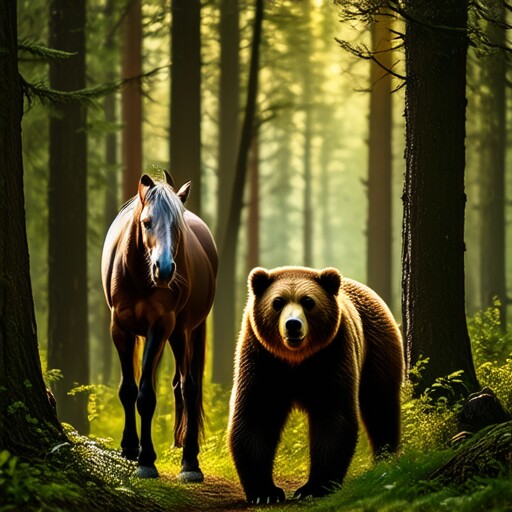} &
        \includegraphics[width=0.1\linewidth]{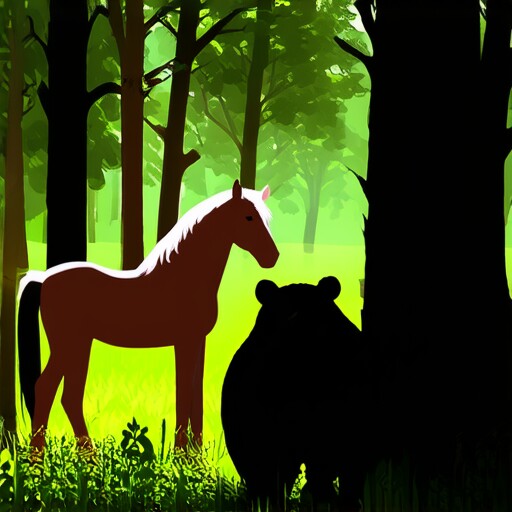} &
        \includegraphics[width=0.1\linewidth]{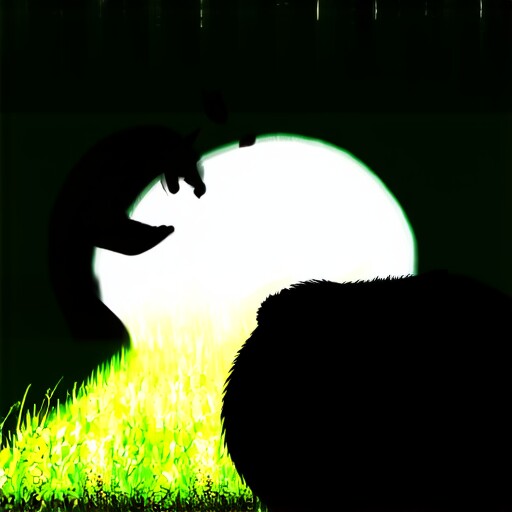}
    \end{tabular}
    }}
    \caption{Effect of the control parameter $\lambda$ on test-time control with Stable Diffusion 3.5.}
    \label{fig:lambda_example}
\end{figure}

\subsection{Fine-tuning Hyperparameters}
We initialize the \emph{memoryless} schedule from each model’s ODE 28-step inference schedule (same time steps), do not use classifier-free guidance, and for \textbf{FLUX.1} apply its native guidance scale (not CFG). Following \cref{app:sampling_parameters}, we cap tokenized sequence length for cross-attention extraction to 77 (SD~3.5) and 256 (FLUX.1). Models are loaded in \texttt{bfloat16}; forward/backward passes run in BF16 and the final loss reduction is computed in FP32 to avoid numerical issues. To reduce memory, at each iteration we subsample 16 of the 28 steps to be used in our loss calculation. We further use a batch sizes of 5 trajectories for SD~3.5 and 2 trajectories for FLUX.1. We use two small prompt sets: 1 (single prompt: “A horse and a bear”) and 15 (each with two semantically similar subjects). Optimization uses AdamW with a weight decay of 0.01 and $\beta_0 = 0.95$, $\beta_1=0.999$. In addition, we also employ \texttt{Accelerate} to lower peak memory consumption. \Cref{tab:hyperparameters} lists the hyperparameter grids we sweep per heuristic; best settings are \textbf{bold}. 

\begin{table*}[thb]
    \centering \small
    \renewcommand{\arraystretch}{1.05}
    \caption{Hyperparameter grids for fine-tuning; best settings per row in \textbf{bold}.}
    \begin{tabular}{@{}c<{\enspace}@{}lcccccccc@{}}
    \toprule
    & \textbf{Heuristic}
    & \textbf{Lambda $\lambda$}
    & \textbf{Learning rate }
    & \textbf{Checkpoint}
    & \textbf{\#Prompts} \\
    \midrule
    \multirow{4}{*}{\rotatebox{90}{\textbf{SD 3.5}}}
    & Attend\&Excite & \{0.1, 1, \textbf{10}\} & $5\mathrm{e}{-5}$ & \{\textbf{100}, 150 \} & 1\\
    & CONFORM & \{\textbf{0.1}, 1, 10\} & $5\mathrm{e}{-5}$ & \{\textbf{100}, 150 \} & 1\\
    & Divide\&Bind & \{\textbf{0.1}, 1, 10\} & $5\mathrm{e}{-5}$ & \{\textbf{100}, 150 \} & 1\\
    & Self-Cross Guidance & \{0.1, 1, \textbf{10}\} & $5\mathrm{e}{-5}$ & \{\textbf{100}, 150 \} & \textbf{15}\\
    & FOCUS & \{0.01, 0.1, \textbf{1}, 10, 100\} & \{$1\mathrm{e}{-4}$, \bm{$5\mathrm{e}{-5}$}, $1\mathrm{e}{-5}$\} & \{\textbf{100}, 150, 200\} & \{\textbf{1}, 15, 150\}\\
    \midrule
    \multirow{4}{*}{\rotatebox{90}{\textbf{FLUX.1}}}
    & Attend\&Excite & \{0.1, 1, \textbf{10}\} & $5\mathrm{e}{-5}$ & \{\textbf{200}, 250 \} & 1\\
    & CONFORM & \{0.1, 1, \textbf{10}\} & $5\mathrm{e}{-5}$ & \{\textbf{200}, 250 \} & 1\\
    & Divide\&Bind & \{\textbf{0.1}, 1, 10\} & $5\mathrm{e}{-5}$ & \{\textbf{200}, 250 \} & 1\\
    & Self-Cross Guidance & \{\textbf{0.1}, 1, 10\} & $5\mathrm{e}{-5}$ & \{\textbf{200}, 250 \} & \{1, \textbf{15}\}\\
    & FOCUS & \{0.01, 0.1, 1, 10, \textbf{100}\} & \{$1\mathrm{e}{-4}$, \bm{$5\mathrm{e}{-5}$}, $1\mathrm{e}{-5}$\} & \{200, \textbf{250}, 300\} & \{\textbf{1}, 15, 150\}\\
    \bottomrule
    \end{tabular}
    \label{tab:hyperparameters}
\end{table*}

\subsection{Additional Metric: Open-Vocabulary Detection}
As a complementary metric, we assess \emph{subject presence} with OWL-V2 open-vocabulary detection \cite{minderer_scaling_2024}. For each prompt, we pass the subject strings as class queries and count an image as correct if \emph{all} subjects are detected at least once. We report the fraction of images meeting this criterion.

Results for test-time control and fine-tuned models are shown in \Cref{tab:OWL_otf,tab:OWL_fine}. Both control algorithms increase subject presence over the base model. However, OWL-V2 is blind to attribute leakage and subject numerosity (it does not verify attributes or counts), so we exclude it from the main evaluation and report it only as a supportive metric here.

\begin{table}[t]
    \centering
    \hfill
    \begin{minipage}[t]{0.45\linewidth}
        \centering
        \caption{OWL-V2 subject presence [\%] under test-time control. For each heuristic, we report the hyperparameter run with the highest composite score, see \cref{tab:metrics_otf_sd3,tab:metrics_otf_flux} for details.}
        \label{tab:OWL_otf}
        \begin{tabular}{@{}lccc@{}}
            \toprule
            \textbf{Heuristic} & \textbf{SD3.5} & \textbf{FLUX}\\
            \midrule
            Base & 69.33\% & 66.93\%\\
            Attend\&Excite & \cellcolor{bronze!20}72.13\% & 66.80\% \\
            CONFORM & \cellcolor{gold!20}\textbf{77.20}\% &\cellcolor{bronze!20} 67.87\% \\
            Divide\&Bind & 70.80\% & \cellcolor{gold!20}\textbf{68.53\%} \\
            FOCUS (Ours) & \cellcolor{silver!20}74.27\% & \cellcolor{silver!20}68.27\% \\
            \bottomrule
        \end{tabular}
    \end{minipage}\hfill
    \begin{minipage}[t]{0.45\linewidth}
      \centering
      \caption{OWL-V2 subject presence [\%] under fine-tuned models. For each heuristic, we report the hyperparameter run with the highest composite score, see \cref{tab:metrics_fine_sd3_1,tab:metrics_fine_sd3_2,tab:metrics_fine_sd3_3,tab:metrics_fine_flux_1,tab:metrics_fine_flux_2,tab:metrics_fine_flux_3} for details.}
      \label{tab:OWL_fine}
      \begin{tabular}{@{}lccc@{}}
            \toprule
            \textbf{Heuristic} & \textbf{SD3.5} & \textbf{FLUX}\\
            \midrule
            Base & 69.33\% & 66.93\%\\
            Attend\&Excite & \cellcolor{gold!20}\textbf{80.40\%} & \cellcolor{gold!20}\textbf{74.93\%}\\
            CONFORM & \cellcolor{bronze!20}77.73\% & \cellcolor{bronze!20}72.53\% \\
            Divide\&Bind & 73.33\% & 63.87\% \\
            FOCUS (Ours) & \cellcolor{silver!20}78.53\% & \cellcolor{silver!20}74.66\% \\
            \bottomrule
        \end{tabular}
    \end{minipage}
    \hfill
\end{table}

\section{Human Study}\label{app:human_study}
We test whether automatic metric gains align with human preferences via a prompt-conditioned, pairwise preference study.

\begin{figure}[htb]
    \centering
    \includegraphics[width=0.75\linewidth]{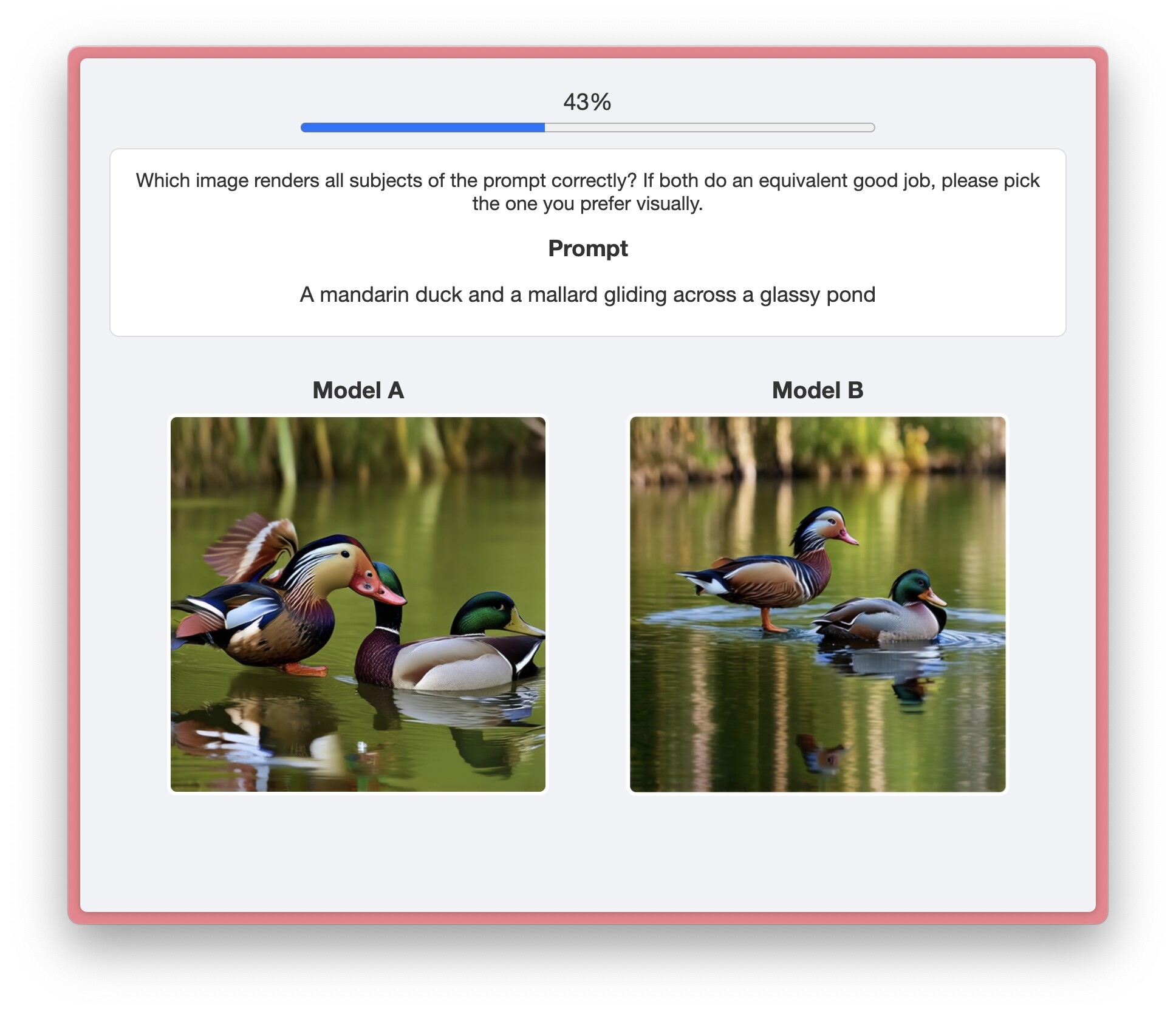}
    \caption{User interface for the prompt-conditioned, pairwise preference study.}
    \label{fig:human_study}
\end{figure}

\subsection{Setup}
We evaluate whether metric gains translate to human preferences. Fifty participants each completed $40$ prompt-conditioned, pairwise trials, resulting in $2{,}000$ total judgments. In every trial, two images generated from the \emph{same} prompt were shown side by side with the prompt; participants selected the image that better matched the prompt. The instruction shown was:
\begin{quote}
    \textit{``Which image renders all subjects of the prompt correctly? If both do an equivalent good job, please pick the one you prefer visually.''}
\end{quote}
To ensure sufficient rating density, we fixed the sampling seed to $0$, yielding one image per method–prompt pair (pool of 150 prompts). Trials were balanced across backbone and setting: SD~3.5 vs.\ FLUX.1 and test-time control vs.\ fine-tuning each accounted for one quarter of the comparisons per participant. A screenshot of the interface is shown in \cref{fig:human_study}.

\subsection{Elo Rating Computation}
We compute Elo ratings from the pairwise outcomes to obtain an across-method ranking, alongside win rates (fraction of pairwise wins). Elo is initialized at $1500$ for all candidates and updated after each comparison with $K{=}32$. For a candidate $A$ with rating $R_A$ matched against $B$ with $R_B$, the expected score is
\begin{align}
    E_A = \frac{1}{1 + 10^{(R_B - R_A)/400}},
\end{align}
and the update is 
\begin{align}
        R_A' = R_A + K \left( S_A - E_A \right)
\end{align}
where $S_A = 1$ for a win, $0$ for a loss, and $0.5$ for a draw. Higher Elo indicates stronger preference relative to alternatives. Win rate is reported as the proportion of head-to-head wins.

\section{StableDiffusion 1.5}
Although our algorithms are derived for flow matching, \cref{app:denoising} shows how classical denoising diffusion can be cast in the same framework. To test transferability, we apply the \emph{test-time controller} to \emph{Stable Diffusion~1.5} (SD~1.5)\footnote{\href{https://huggingface.co/stable-diffusion-v1-5/stable-diffusion-v1-5}{https://huggingface.co/stable-diffusion-v1-5/stable-diffusion-v1-5}} \cite{rombach_high-resolution_2022}, a U-Net–based denoising diffusion model, and outline the minimal implementation changes below.

\subsection{Implementation: U-Net Architecture}
Diffusion Transformer (DiT) backbones treat image and text as token sequences processed by stacked Transformer blocks; image–text interaction arises via self-attention of a fixed spatial size. In SD~1.5, based on a U-Net, the latent is downsampled and upsampled through multiple stages, and text conditioning is injected via cross-attention at several resolutions. Consequently, cross-attention maps have \emph{different spatial sizes} across the network.

To obtain a single subject-specific map per prompt step, we collect \emph{all} cross-attention maps (across down/upsampling and bottleneck), bilinearly resize each to $16\times16$, and average them (over heads and layers). This preserves signal from every stage while standardizing spatial shape. Prior works often use only a fixed-resolution subset (\eg{,} bottleneck) \cite{meral_conform_2023,chefer_attend-and-excite_2023}; we found the all-maps aggregation simpler and more comparable across methods, which is the goal of this transfer study.

\subsection{Empirical Results}
We reuse the main evaluation pipeline for SD~1.5, sweeping $10$ values of $\lambda$ and using the same update schedule as in \Cref{eq:impl_update}. A qualitative example is shown in \Cref{fig:otf_sd1_teaser}, with more in \Cref{fig:otf_sd1}. Metric results at each heuristic’s best $\lambda$ appear in \Cref{tab:metrics_otf_sd1_teaser}, with the full sweep in \Cref{tab:metrics_otf_sd1}.

Across metrics and prompts, all test-time heuristics improve over the base SD~1.5 model, confirming that our formulation transfers to denoising diffusion. \textsc{FOCUS} achieves the highest composite score and ranks among the top methods on most individual metrics, while preserving the base style and reducing attribute leakage in the qualitative results.

\begin{figure}[htb]
  \centering
  {\setlength{\tabcolsep}{.3mm}%
  \resizebox{0.95\linewidth}{!}{%
    \begin{tabular}{@{}ccccccc@{}}
      \textbf{Base} & \textbf{Attend\&Excite} & \textbf{CONFORM} & \textbf{Divide\&Bind} & \textbf{FOCUS (Ours)} \\
      \QUINTAND{SD1-OTF}{1}\\
    \end{tabular}%
  }}
    \caption{Stable Diffusion~1.5 samples with test-time control. All heuristics shown at their optimal $lambda$. The prompt is ``A red fox and an arctic fox sitting side by side in tall grass''.}
  \label{fig:otf_sd1_teaser}
\end{figure}

\begin{table*}[t]
    \caption{Test-time control results for Stable Diffusion~1.5. We report mean $\pm$ std over all prompts and seeds; the top three values per metric are highlighted (gold/silver/bronze). Each method uses the same sampling/evaluation pipeline and its optimal $\lambda$.}
    \vspace{-1.5mm}
    \label{tab:metrics_otf_sd1_teaser}
    \centering
    \resizebox{\textwidth}{!}{
    \renewcommand{\arraystretch}{1.05}
    \setlength{\tabcolsep}{1.5mm}%
	\begin{tabular}{@{}c<{\enspace}@{}lcccccc|cc@{}}\toprule
    & \textbf{Heuristic} & \textbf{CLIP  I-T$\uparrow$} & \textbf{SigLIP-2 I-T$\uparrow$} & \textbf{BLIP T-T$\uparrow$} & \textbf{Qwen2 T-T$\uparrow$} & \textbf{PickScore I-T$\uparrow$} & \textbf{ImgRew I-T$\uparrow$} & \textbf{Composite$\uparrow$} \\
    \midrule
    \multirow{5}{*}{\rotatebox{90}{\textbf{SD 1.5}}}
    & Base & 0.3304\scriptsize$\pm$0.03 & 0.1892\scriptsize$\pm$0.04 & \cellcolor{bronze!20}0.5433\scriptsize$\pm$0.14 & 0.5742\scriptsize$\pm$0.09 & 21.2563\scriptsize$\pm$0.99 & -0.0763\scriptsize$\pm$1.00 & 0.0000\scriptsize$\pm$0.00 \\
    & Attend\&Excite & 0.3321\scriptsize$\pm$0.03 & \cellcolor{gold!20}\textbf{0.1964}\scriptsize$\pm$0.04 & 0.5415\scriptsize$\pm$0.15 & 0.5803\scriptsize$\pm$0.09 & 21.2082\scriptsize$\pm$1.02 & 0.0491\scriptsize$\pm$1.05 & 780.5529\scriptsize$\pm$22303.71 \\
    & CONFORM & \cellcolor{silver!20}0.3348\scriptsize$\pm$0.03 & \cellcolor{bronze!20}0.1954\scriptsize$\pm$0.04 & \cellcolor{gold!20}\textbf{0.5449}\scriptsize$\pm$0.14 & \cellcolor{silver!20}0.5834\scriptsize$\pm$0.09 & \cellcolor{gold!20}\textbf{21.3235}\scriptsize$\pm$0.98 & \cellcolor{silver!20}0.1137\scriptsize$\pm$0.99 & \cellcolor{silver!20}1842.0251\scriptsize$\pm$50959.48 \\
    & Divide\&Bind & \cellcolor{gold!20}\textbf{0.3350}\scriptsize$\pm$0.03 & \cellcolor{silver!20}0.1961\scriptsize$\pm$0.04 & \cellcolor{silver!20}0.5441\scriptsize$\pm$0.14 & \cellcolor{bronze!20}0.5817\scriptsize$\pm$0.09 & \cellcolor{silver!20}21.3085\scriptsize$\pm$0.99 & \cellcolor{bronze!20}0.1081\scriptsize$\pm$1.02 & \cellcolor{bronze!20}1794.9343\scriptsize$\pm$49951.59 \\
    & FOCUS(Ours) & \cellcolor{bronze!20}0.3327\scriptsize$\pm$0.03 & 0.1952\scriptsize$\pm$0.04 & 0.5429\scriptsize$\pm$0.14 & \cellcolor{gold!20}\textbf{0.5834}\scriptsize$\pm$0.09 & \cellcolor{bronze!20}21.2977\scriptsize$\pm$0.99 & \cellcolor{gold!20}\textbf{0.1173}\scriptsize$\pm$1.02 & \cellcolor{gold!20}\textbf{1862.0311}\scriptsize$\pm$51706.75 \\
    \bottomrule
  \end{tabular}}
\end{table*}

\section{Self–Cross Guidance (SCG)}
Self–Cross Guidance (SCG) \cite{qiu2025self} is, to our knowledge, the only prior heuristic explicitly developed for multi-subject disentanglement on modern DiT-based T2I models (\eg{,} SD~3/3.5). We therefore include a careful implementation and comparison to \textsc{FOCUS}.

Unlike cross-attention–only objectives, SCG exploits \emph{two} internal signals: (i) text-to-image cross-attention and (ii) image-to-image self-attention. Intuitively, the self-attention term is used to decorrelate image features associated with different subject tokens while the cross-attention term promotes subject binding. This additional signal distinguishes SCG from other heuristics we evaluate.

\subsection{Evaluation in Our Pipeline}
Under our standard test-time control protocol, SCG marginally exceeds \textsc{FOCUS} on the composite score for SD~3.5. Qualitatively, however, we frequently observe side effects consistent with stronger separation pressure: a tendency toward stylized (cartoon-like) textures and occasional numerosity artifacts (\eg{,} producing extra instances rather than cleanly separating two subjects); see \Cref{fig:otf_scg_comparison} and additional examples in \Cref{fig:fine_flux,fig:fine_sd3,fig:otf_flux,fig:otf_sd3}. This observation is consistent with the limitations mentioned by the authors of SCG. By contrast, \textsc{FOCUS} preserves base style more reliably and maintains subject counts, and achieves higher scores than SCG after fine-tuning.

\subsection{Evaluation on the SCG Dataset}
For comparability with \cite{qiu2025self}, we also evaluate on their released prompt suite\footnote{\href{https://github.com/mengtang-lab/selfcross-guidance/blob/main/prompts.txt}{https://github.com/mengtang-lab/selfcross-guidance/blob/main/prompts.txt}} comprising five subsets: SSD-3 (3 similar subjects; 22 prompts), SSD-2 (2 similar subjects; 31 prompts), Animal–Animal (66 prompts), Animal–Object (144 prompts), and Object–Object (66 prompts). Prompts follow fixed templates such as ``a \textit{SUBJECT~A} and a \textit{SUBJECT~B}'' or ``a \textit{SUBJECT~A} with an \textit{OBJECT~B}.''

\paragraph{Test-time control (SD~3.5).} Using our standard 10-point $\lambda$ sweep, both SCG and \textsc{FOCUS} consistently improve over the base model, see \cref{tab:metrics_otf_scg}. SCG leads on SSD-2, while \textsc{FOCUS} attains higher scores on the remaining subsets, yielding overall comparable performance with a slight average advantage for \textsc{FOCUS}.

\begin{table*}[htb]
    \caption{Test-time control results on the SCG Prompt Dataset. We report mean $\pm$ std over all prompts and seeds; the top three values per metric are highlighted (gold/silver/bronze). Each method uses the same sampling/evaluation pipeline and its optimal $\lambda$.}
    \vspace{-1.5mm}
    \label{tab:metrics_otf_scg}
    \centering
    \resizebox{\textwidth}{!}{
    \renewcommand{\arraystretch}{1.05}
    \setlength{\tabcolsep}{1.5mm}%
	\begin{tabular}{@{}l<{\enspace}@{}l|cccccc|cc@{}}\toprule
    \textbf{Dataset} & \textbf{Heuristic} & \textbf{CLIP  I-T$\uparrow$} & \textbf{SigLIP-2 I-T$\uparrow$} & \textbf{BLIP T-T$\uparrow$} & \textbf{Qwen2 T-T$\uparrow$} & \textbf{PickScore I-T$\uparrow$} & \textbf{ImgRew I-T$\uparrow$} & \textbf{Composite$\uparrow$} \\
    \midrule
    \multirow{3}{*}{\textbf{SSD-3}}  & Base & \cellcolor{bronze!20}0.6494\scriptsize$\pm$0.01 & \cellcolor{bronze!20}0.5935\scriptsize$\pm$0.01 & \cellcolor{silver!20}0.7886\scriptsize$\pm$0.07 & \cellcolor{bronze!20}0.7996\scriptsize$\pm$0.04 & \cellcolor{silver!20}0.2919\scriptsize$\pm$0.11 & \cellcolor{bronze!20}0.7464\scriptsize$\pm$0.19 & \cellcolor{bronze!20}0.0000\scriptsize$\pm$0.00 \\
    & Self-Cross Guidance & \cellcolor{silver!20}0.6504\scriptsize$\pm$0.01 & \cellcolor{silver!20}0.5946\scriptsize$\pm$0.01 & \cellcolor{gold!20}\textbf{0.7969}\scriptsize$\pm$0.07 & \cellcolor{gold!20}\textbf{0.8035}\scriptsize$\pm$0.04 & \cellcolor{bronze!20}0.2914\scriptsize$\pm$0.11 & \cellcolor{silver!20}0.7689\scriptsize$\pm$0.18 & \cellcolor{silver!20}60.6311\scriptsize$\pm$394.49 \\
    & FOCUS(Ours) & \cellcolor{gold!20}\textbf{0.6506}\scriptsize$\pm$0.01 & \cellcolor{gold!20}\textbf{0.5957}\scriptsize$\pm$0.01 & \cellcolor{bronze!20}0.7882\scriptsize$\pm$0.08 & \cellcolor{silver!20}0.7999\scriptsize$\pm$0.04 & \cellcolor{gold!20}\textbf{0.2928}\scriptsize$\pm$0.12 & \cellcolor{gold!20}\textbf{0.7915}\scriptsize$\pm$0.15 & \cellcolor{gold!20}\textbf{82.1554}\scriptsize$\pm$455.15 \\
    \midrule
    \multirow{3}{*}{\textbf{SSD-2}}  & Base & \cellcolor{bronze!20}0.6477\scriptsize$\pm$0.01 & \cellcolor{bronze!20}0.5888\scriptsize$\pm$0.02 & \cellcolor{bronze!20}0.7570\scriptsize$\pm$0.08 & \cellcolor{bronze!20}0.7904\scriptsize$\pm$0.05 & \cellcolor{bronze!20}0.2495\scriptsize$\pm$0.10 & \cellcolor{bronze!20}0.6620\scriptsize$\pm$0.19 & \cellcolor{bronze!20}0.0000\scriptsize$\pm$0.00 \\
    & Self-Cross Guidance & \cellcolor{silver!20}0.6500\scriptsize$\pm$0.01 & \cellcolor{silver!20}0.5947\scriptsize$\pm$0.01 & \cellcolor{silver!20}0.7598\scriptsize$\pm$0.07 & \cellcolor{gold!20}\textbf{0.7958}\scriptsize$\pm$0.05 & \cellcolor{gold!20}\textbf{0.2564}\scriptsize$\pm$0.11 & \cellcolor{silver!20}0.7277\scriptsize$\pm$0.17 & \cellcolor{gold!20}\textbf{148.2953}\scriptsize$\pm$447.50 \\
    & FOCUS(Ours) & \cellcolor{gold!20}\textbf{0.6505}\scriptsize$\pm$0.01 & \cellcolor{gold!20}\textbf{0.5958}\scriptsize$\pm$0.01 & \cellcolor{gold!20}\textbf{0.7625}\scriptsize$\pm$0.08 & \cellcolor{silver!20}0.7926\scriptsize$\pm$0.05 & \cellcolor{silver!20}0.2531\scriptsize$\pm$0.11 & \cellcolor{gold!20}\textbf{0.7279}\scriptsize$\pm$0.17 & \cellcolor{silver!20}144.9927\scriptsize$\pm$522.25 \\
    \midrule
    \multirow{3}{*}{\textbf{Animal---Animal}}  & Base & \cellcolor{bronze!20}0.6580\scriptsize$\pm$0.01 & \cellcolor{bronze!20}0.6036\scriptsize$\pm$0.01 & \cellcolor{bronze!20}0.8656\scriptsize$\pm$0.07 & \cellcolor{bronze!20}0.8156\scriptsize$\pm$0.03 & \cellcolor{bronze!20}0.4239\scriptsize$\pm$0.12 & \cellcolor{bronze!20}0.9020\scriptsize$\pm$0.13 & \cellcolor{bronze!20}0.0000\scriptsize$\pm$0.00 \\
    & Self-Cross Guidance & \cellcolor{silver!20}0.6583\scriptsize$\pm$0.01 & \cellcolor{silver!20}0.6039\scriptsize$\pm$0.01 & \cellcolor{gold!20}\textbf{0.8693}\scriptsize$\pm$0.07 & \cellcolor{gold!20}\textbf{0.8170}\scriptsize$\pm$0.03 & \cellcolor{gold!20}\textbf{0.4316}\scriptsize$\pm$0.11 & \cellcolor{silver!20}0.9101\scriptsize$\pm$0.12 & \cellcolor{silver!20}35.9864\scriptsize$\pm$343.01 \\
    & FOCUS(Ours) & \cellcolor{gold!20}\textbf{0.6586}\scriptsize$\pm$0.01 & \cellcolor{gold!20}\textbf{0.6041}\scriptsize$\pm$0.01 & \cellcolor{silver!20}0.8667\scriptsize$\pm$0.07 & \cellcolor{silver!20}0.8168\scriptsize$\pm$0.03 & \cellcolor{silver!20}0.4298\scriptsize$\pm$0.12 & \cellcolor{gold!20}\textbf{0.9177}\scriptsize$\pm$0.10 & \cellcolor{gold!20}\textbf{41.9020}\scriptsize$\pm$418.06 \\
    \midrule
    \multirow{3}{*}{\textbf{Animal---Object}}  & Base & \cellcolor{gold!20}\textbf{0.6710}\scriptsize$\pm$0.01 & \cellcolor{bronze!20}0.6101\scriptsize$\pm$0.01 & \cellcolor{bronze!20}0.8941\scriptsize$\pm$0.06 & \cellcolor{bronze!20}0.8415\scriptsize$\pm$0.04 & \cellcolor{bronze!20}0.5234\scriptsize$\pm$0.15 & \cellcolor{bronze!20}0.9344\scriptsize$\pm$0.12 & \cellcolor{bronze!20}0.0000\scriptsize$\pm$0.00 \\
    & Self-Cross Guidance & \cellcolor{bronze!20}0.6708\scriptsize$\pm$0.01 & \cellcolor{gold!20}\textbf{0.6101}\scriptsize$\pm$0.01 & \cellcolor{silver!20}0.8954\scriptsize$\pm$0.06 & \cellcolor{gold!20}\textbf{0.8444}\scriptsize$\pm$0.03 & \cellcolor{silver!20}0.5293\scriptsize$\pm$0.14 & \cellcolor{silver!20}0.9361\scriptsize$\pm$0.12 & \cellcolor{silver!20}19.2260\scriptsize$\pm$254.76 \\
    & FOCUS(Ours) & \cellcolor{silver!20}0.6710\scriptsize$\pm$0.01 & \cellcolor{silver!20}0.6101\scriptsize$\pm$0.01 & \cellcolor{gold!20}\textbf{0.8983}\scriptsize$\pm$0.05 & \cellcolor{silver!20}0.8440\scriptsize$\pm$0.03 & \cellcolor{gold!20}\textbf{0.5308}\scriptsize$\pm$0.14 & \cellcolor{gold!20}\textbf{0.9416}\scriptsize$\pm$0.11 & \cellcolor{gold!20}\textbf{35.3828}\scriptsize$\pm$317.06 \\
    \midrule
    \multirow{3}{*}{\textbf{Object---Object}}  & Base & \cellcolor{bronze!20}0.6754\scriptsize$\pm$0.01 & \cellcolor{bronze!20}0.6136\scriptsize$\pm$0.02 & \cellcolor{bronze!20}0.9003\scriptsize$\pm$0.06 & \cellcolor{bronze!20}0.8519\scriptsize$\pm$0.04 & \cellcolor{bronze!20}0.5409\scriptsize$\pm$0.16 & \cellcolor{bronze!20}0.9197\scriptsize$\pm$0.15 & \cellcolor{bronze!20}0.0000\scriptsize$\pm$0.00 \\
    & Self-Cross Guidance & \cellcolor{silver!20}0.6757\scriptsize$\pm$0.01 & \cellcolor{gold!20}\textbf{0.6158}\scriptsize$\pm$0.02 & \cellcolor{silver!20}0.9053\scriptsize$\pm$0.06 & \cellcolor{silver!20}0.8543\scriptsize$\pm$0.04 & \cellcolor{silver!20}0.5414\scriptsize$\pm$0.16 & \cellcolor{silver!20}0.9320\scriptsize$\pm$0.12 & \cellcolor{silver!20}37.8421\scriptsize$\pm$507.36 \\
    & FOCUS(Ours) & \cellcolor{gold!20}\textbf{0.6762}\scriptsize$\pm$0.01 & \cellcolor{silver!20}0.6154\scriptsize$\pm$0.02 & \cellcolor{gold!20}\textbf{0.9060}\scriptsize$\pm$0.06 & \cellcolor{gold!20}\textbf{0.8562}\scriptsize$\pm$0.04 & \cellcolor{gold!20}\textbf{0.5529}\scriptsize$\pm$0.16 & \cellcolor{gold!20}\textbf{0.9380}\scriptsize$\pm$0.11 & \cellcolor{gold!20}\textbf{71.5782}\scriptsize$\pm$434.59 \\
    \bottomrule
  \end{tabular}}
\end{table*}

\begin{table*}[htb]
    \caption{Fine-tuned models evaluate on the SCG Prompt Dataset. We report mean $\pm$ std over all prompts and seeds; the top three values per metric are highlighted (gold/silver/bronze). Each method uses the same sampling/evaluation pipeline and its optimal $\lambda$.}
    \vspace{-1.5mm}
    \label{tab:metrics_fine_scg}
    \centering
    \resizebox{\textwidth}{!}{
    \renewcommand{\arraystretch}{1.05}
    \setlength{\tabcolsep}{1.5mm}%
	\begin{tabular}{@{}l<{\enspace}@{}l|cccccc|cc@{}}\toprule
    \textbf{Dataset} & \textbf{Heuristic} & \textbf{CLIP  I-T$\uparrow$} & \textbf{SigLIP-2 I-T$\uparrow$} & \textbf{BLIP T-T$\uparrow$} & \textbf{Qwen2 T-T$\uparrow$} & \textbf{PickScore I-T$\uparrow$} & \textbf{ImgRew I-T$\uparrow$} & \textbf{Composite$\uparrow$} \\
    \midrule
    \multirow{3}{*}{\textbf{SSD-3}}  & Base & \cellcolor{bronze!20}0.6494\scriptsize$\pm$0.01 & \cellcolor{bronze!20}0.5935\scriptsize$\pm$0.01 & \cellcolor{bronze!20}0.7886\scriptsize$\pm$0.07 & \cellcolor{bronze!20}0.7996\scriptsize$\pm$0.04 & \cellcolor{silver!20}0.2919\scriptsize$\pm$0.11 & \cellcolor{bronze!20}0.7464\scriptsize$\pm$0.19 & \cellcolor{bronze!20}0.0000\scriptsize$\pm$0.00 \\
    & Self-Cross Guidance & \cellcolor{silver!20}0.6508\scriptsize$\pm$0.01 & \cellcolor{silver!20}0.5977\scriptsize$\pm$0.01 & \cellcolor{silver!20}0.7931\scriptsize$\pm$0.07 & \cellcolor{silver!20}0.8014\scriptsize$\pm$0.04 & \cellcolor{bronze!20}0.2873\scriptsize$\pm$0.11 & \cellcolor{silver!20}0.8139\scriptsize$\pm$0.14 & \cellcolor{silver!20}124.4542\scriptsize$\pm$373.52 \\
    & FOCUS(Ours) & \cellcolor{gold!20}\textbf{0.6526}\scriptsize$\pm$0.01 & \cellcolor{gold!20}\textbf{0.5987}\scriptsize$\pm$0.01 & \cellcolor{gold!20}\textbf{0.7936}\scriptsize$\pm$0.07 & \cellcolor{gold!20}\textbf{0.8069}\scriptsize$\pm$0.03 & \cellcolor{gold!20}\textbf{0.2957}\scriptsize$\pm$0.12 & \cellcolor{gold!20}\textbf{0.8301}\scriptsize$\pm$0.12 & \cellcolor{gold!20}\textbf{180.3947}\scriptsize$\pm$451.87 \\
    \midrule
    \multirow{3}{*}{\textbf{SSD-2}}  & Base & \cellcolor{bronze!20}0.6477\scriptsize$\pm$0.01 & \cellcolor{bronze!20}0.5888\scriptsize$\pm$0.02 & \cellcolor{bronze!20}0.7570\scriptsize$\pm$0.08 & \cellcolor{bronze!20}0.7904\scriptsize$\pm$0.05 & \cellcolor{bronze!20}0.2495\scriptsize$\pm$0.10 & \cellcolor{bronze!20}0.6620\scriptsize$\pm$0.19 & \cellcolor{bronze!20}0.0000\scriptsize$\pm$0.00 \\
    & Self-Cross Guidance & \cellcolor{silver!20}0.6500\scriptsize$\pm$0.01 & \cellcolor{silver!20}0.5976\scriptsize$\pm$0.01 & \cellcolor{silver!20}0.7666\scriptsize$\pm$0.07 & \cellcolor{gold!20}\textbf{0.8006}\scriptsize$\pm$0.05 & \cellcolor{silver!20}0.2520\scriptsize$\pm$0.11 & \cellcolor{silver!20}0.7578\scriptsize$\pm$0.14 & \cellcolor{silver!20}215.1917\scriptsize$\pm$500.96 \\
    & FOCUS(Ours) & \cellcolor{gold!20}\textbf{0.6506}\scriptsize$\pm$0.01 & \cellcolor{gold!20}\textbf{0.5988}\scriptsize$\pm$0.01 & \cellcolor{gold!20}\textbf{0.7775}\scriptsize$\pm$0.07 & \cellcolor{silver!20}0.8000\scriptsize$\pm$0.06 & \cellcolor{gold!20}\textbf{0.2523}\scriptsize$\pm$0.10 & \cellcolor{gold!20}\textbf{0.7618}\scriptsize$\pm$0.13 & \cellcolor{gold!20}\textbf{242.7119}\scriptsize$\pm$499.00 \\
    \midrule
    \multirow{3}{*}{\textbf{Animal---Animal}}  & Base & \cellcolor{silver!20}0.6580\scriptsize$\pm$0.01 & \cellcolor{bronze!20}0.6036\scriptsize$\pm$0.01 & \cellcolor{bronze!20}0.8656\scriptsize$\pm$0.07 & \cellcolor{bronze!20}0.8156\scriptsize$\pm$0.03 & \cellcolor{silver!20}0.4239\scriptsize$\pm$0.12 & \cellcolor{bronze!20}0.9020\scriptsize$\pm$0.13 & \cellcolor{bronze!20}0.0000\scriptsize$\pm$0.00 \\
    & Self-Cross Guidance & \cellcolor{gold!20}\textbf{0.6583}\scriptsize$\pm$0.01 & \cellcolor{gold!20}\textbf{0.6054}\scriptsize$\pm$0.01 & \cellcolor{gold!20}\textbf{0.8803}\scriptsize$\pm$0.05 & \cellcolor{silver!20}0.8236\scriptsize$\pm$0.03 & \cellcolor{gold!20}\textbf{0.4334}\scriptsize$\pm$0.11 & \cellcolor{gold!20}\textbf{0.9345}\scriptsize$\pm$0.05 & \cellcolor{gold!20}\textbf{111.3952}\scriptsize$\pm$372.56 \\
    & FOCUS(Ours) & \cellcolor{bronze!20}0.6576\scriptsize$\pm$0.01 & \cellcolor{silver!20}0.6039\scriptsize$\pm$0.01 & \cellcolor{silver!20}0.8754\scriptsize$\pm$0.06 & \cellcolor{gold!20}\textbf{0.8240}\scriptsize$\pm$0.03 & \cellcolor{bronze!20}0.4003\scriptsize$\pm$0.12 & \cellcolor{silver!20}0.9230\scriptsize$\pm$0.06 & \cellcolor{silver!20}25.9825\scriptsize$\pm$424.15 \\
    \midrule
    \multirow{3}{*}{\textbf{Animal---Object}}  & Base & \cellcolor{silver!20}0.6710\scriptsize$\pm$0.01 & \cellcolor{bronze!20}0.6101\scriptsize$\pm$0.01 & \cellcolor{bronze!20}0.8941\scriptsize$\pm$0.06 & \cellcolor{bronze!20}0.8415\scriptsize$\pm$0.04 & \cellcolor{gold!20}\textbf{0.5234}\scriptsize$\pm$0.15 & \cellcolor{bronze!20}0.9344\scriptsize$\pm$0.12 & \cellcolor{silver!20}0.0000\scriptsize$\pm$0.00 \\
    & Self-Cross Guidance & \cellcolor{bronze!20}0.6698\scriptsize$\pm$0.01 & \cellcolor{silver!20}0.6104\scriptsize$\pm$0.01 & \cellcolor{gold!20}\textbf{0.8972}\scriptsize$\pm$0.05 & \cellcolor{silver!20}0.8443\scriptsize$\pm$0.03 & \cellcolor{silver!20}0.5029\scriptsize$\pm$0.14 & \cellcolor{silver!20}0.9509\scriptsize$\pm$0.08 & \cellcolor{gold!20}\textbf{1.3660}\scriptsize$\pm$356.49 \\
    & FOCUS(Ours) & \cellcolor{gold!20}\textbf{0.6710}\scriptsize$\pm$0.01 & \cellcolor{gold!20}\textbf{0.6110}\scriptsize$\pm$0.01 & \cellcolor{silver!20}0.8960\scriptsize$\pm$0.06 & \cellcolor{gold!20}\textbf{0.8445}\scriptsize$\pm$0.03 & \cellcolor{bronze!20}0.4888\scriptsize$\pm$0.15 & \cellcolor{gold!20}\textbf{0.9510}\scriptsize$\pm$0.08 & \cellcolor{bronze!20}-20.1645\scriptsize$\pm$382.79 \\
    \midrule
    \multirow{3}{*}{\textbf{Object---Object}}  & Base & \cellcolor{silver!20}0.6754\scriptsize$\pm$0.01 & \cellcolor{bronze!20}0.6136\scriptsize$\pm$0.02 & \cellcolor{bronze!20}0.9003\scriptsize$\pm$0.06 & \cellcolor{silver!20}0.8519\scriptsize$\pm$0.04 & \cellcolor{gold!20}\textbf{0.5409}\scriptsize$\pm$0.16 & \cellcolor{bronze!20}0.9197\scriptsize$\pm$0.15 & \cellcolor{silver!20}0.0000\scriptsize$\pm$0.00 \\
    & Self-Cross Guidance & \cellcolor{bronze!20}0.6742\scriptsize$\pm$0.01 & \cellcolor{gold!20}\textbf{0.6149}\scriptsize$\pm$0.02 & \cellcolor{silver!20}0.9074\scriptsize$\pm$0.06 & \cellcolor{bronze!20}0.8518\scriptsize$\pm$0.04 & \cellcolor{bronze!20}0.4943\scriptsize$\pm$0.14 & \cellcolor{silver!20}0.9449\scriptsize$\pm$0.10 & \cellcolor{bronze!20}-23.8872\scriptsize$\pm$463.48 \\
    & FOCUS(Ours) & \cellcolor{gold!20}\textbf{0.6759}\scriptsize$\pm$0.01 & \cellcolor{silver!20}0.6147\scriptsize$\pm$0.02 & \cellcolor{gold!20}\textbf{0.9111}\scriptsize$\pm$0.06 & \cellcolor{gold!20}\textbf{0.8568}\scriptsize$\pm$0.03 & \cellcolor{silver!20}0.5152\scriptsize$\pm$0.14 & \cellcolor{gold!20}\textbf{0.9494}\scriptsize$\pm$0.08 & \cellcolor{gold!20}\textbf{35.6959}\scriptsize$\pm$459.01 \\
    \bottomrule
  \end{tabular}}
\end{table*}

\paragraph{Fine-tuning generalization.} We evaluate the two best SD~3.5 fine-tuned checkpoints trained on our data, and test them \emph{unchanged} on the SCG prompts, see \cref{tab:metrics_fine_scg}. Both checkpoints generalize and surpass the base model on most subsets. \textsc{FOCUS} underperforms the base model on Animal–Object, whereas SCG is weaker on Object–Object; however, on the general multi-subject splits (SSD-3/SSD-2) both fine-tuned models outperform their test-time counterparts, with SCG leading on Animal–Animal/Animal–Object and \textsc{FOCUS} leading on SSD-3/SSD-2/Object–Object. Additionally, this confirms that our fine-tuned method generalizes well beyond the training dataset, since both checkpoints where achieved with subsets of our own dataset.

\begin{figure}[thb]
    \centering
    \begin{subfigure}[t]{0.49\textwidth}
        {\setlength{\tabcolsep}{.1mm}%
        \resizebox{\linewidth}{!}{%
        \begin{tabular}{@{}cccc@{}}
            \textbf{Base SD 3.5} & \textbf{Self-Cross Guidance} & \textbf{FOCUS (Ours)}\\
            \TRIPLE{SD3-OTF-SCG}{0}\\\TRIPLE{SD3-OTF-SCG}{8}\\
        \end{tabular}}}
        \caption{Test-time control samples on Stable Diffusion 3.5 for the prompt \textit{``A black bear and a brown bear ambling along a riverbank''} and \textit{``A red car, a blue car, and a green car parked side by side on a city street''}.}
    \end{subfigure}
    \hfill
    \begin{subfigure}[t]{0.49\textwidth}
        {\setlength{\tabcolsep}{.1mm}%
        \resizebox{\linewidth}{!}{%
        \begin{tabular}{@{}cccc@{}}
            \textbf{Base FLUX.1} & \textbf{Self-Cross Guidance} & \textbf{FOCUS (Ours)}\\
            \TRIPLE{FLUX-OTF-SCG}{5}\\\TRIPLE{FLUX-OTF-SCG}{1}\\
        \end{tabular}}}
        \caption{Test-time control samples on FLUX.1 [dev] for the prompt \textit{``A fedora, a beanie, and a baseball cap hanging on a coat rack''} and \textit{``A swan, a goose, and a duck drifting past lily pads''}.}
    \end{subfigure}\\\vspace{2mm}
    \begin{subfigure}[t]{0.49\textwidth}
        {\setlength{\tabcolsep}{.1mm}%
        \resizebox{\linewidth}{!}{%
        \begin{tabular}{@{}cccc@{}}
            \textbf{Base SD 3.5} & \textbf{Self-Cross Guidance} & \textbf{FOCUS (Ours)}\\
            \TRIPLE{SD3-FINE-SCG}{7}\\\TRIPLE{SD3-FINE-SCG}{6}\\
        \end{tabular}}}
        \caption{Fine-tuned samples on Stable Diffusion 3.5 for the prompt \textit{``A Labrador, a Golden Retriever, and a German Shepherd playing in a backyard''} and \textit{``A jaguar and a leopard crouching in dense rainforest foliage''}.}
    \end{subfigure}
    \hfill
    \begin{subfigure}[t]{0.49\textwidth}
        {\setlength{\tabcolsep}{.1mm}%
        \resizebox{\linewidth}{!}{%
        \begin{tabular}{@{}cccc@{}}
            \textbf{Base FLUX.1} & \textbf{Self-Cross Guidance} & \textbf{FOCUS (Ours)}\\
            \TRIPLE{FLUX-FINE-SCG}{1}\\\TRIPLE{FLUX-FINE-SCG}{2}\\
        \end{tabular}}}
        \caption{Fine-tuned samples on FLUX.1 [dev] for the prompt \textit{``A macaw, a cockatoo, and an Amazon parrot perched on a jungle vine''} and \textit{``A red fox and an arctic fox sitting side by side in tall grass''}.}
    \end{subfigure}
  \caption{Qualitative comparison of Self–Cross Guidance (SCG) and \textsc{FOCUS}. Rows: test-time control (top) and fine-tuned models (bottom). Columns: Stable Diffusion~3.5 (left) and FLUX.1~[dev] (right).}
  \label{fig:otf_scg_comparison}
\end{figure}

\section{Extra Samples}
\subsection{Test-Time Control: Stable Diffusion 3.5}
\begin{figure}[H]
  \centering
  {\setlength{\tabcolsep}{.3mm}%
  \resizebox{0.98\linewidth}{!}{%
    \begin{tabular}{@{}ccccccc@{}}
        \textbf{Base} & \textbf{Attend\&Excite} & \textbf{CONFORM} & \textbf{Divide 
        \&Bind} & \textbf{Self-Cross Guidance} & \textbf{FOCUS (Ours)} \\
        \HEXAND{SD3-OTF}{0}\\
        \multicolumn{6}{c}{\itshape``A puffin and a penguin standing on a windswept shoreline''}\vspace{1mm}\\
        \HEXAND{SD3-OTF}{5}\\
        \multicolumn{6}{c}{\itshape``A fox, a lantern, and a teapot in a misty forest clearing''}\vspace{1mm}\\
        \HEXAND{SD3-OTF}{9}\\
        \multicolumn{6}{c}{\itshape``A jellyfish, a lighthouse, and a pocket watch suspended in seawater''}\vspace{1mm}\\
        \HEXAND{SD3-OTF}{7}\\
        \multicolumn{6}{c}{\itshape``A sailboat, a bicycle, and a stack of books beside a canal''}\vspace{1mm}\\
        \HEXAND{SD3-OTF}{14}\\
        \multicolumn{6}{c}{\itshape``A bluetit, a croissant, and a porcelain cup on a balcony rail''}\vspace{1mm}\\
        \HEXAND{SD3-OTF}{16}\\
        \multicolumn{5}{c}{\itshape``A violin, a raven, and a pocket watch on a stone windowsill''}\vspace{1mm}\\
    \end{tabular}%
  }}
  \caption{Stable Diffusion 3.5 samples with test-time control. All evaluated heuristics are shown at their optimal $\lambda$.}
  \label{fig:otf_sd3}
\end{figure}

\subsection{Test-Time Control: FLUX.1 [dev]}
\begin{figure}[H]
  \centering
  {\setlength{\tabcolsep}{.3mm}%
  \resizebox{0.98\linewidth}{!}{%
    \begin{tabular}{@{}ccccccc@{}}
      \textbf{Base} & \textbf{Attend\&Excite} & \textbf{CONFORM} & \textbf{Divide 
      \&Bind} & \textbf{Self-Cross Guidance} & \textbf{FOCUS (Ours)} \\
      \HEXAND{FLUX-OTF}{6}\\
      \multicolumn{6}{c}{\itshape``A chameleon, a wristwatch, and a paper crane on a mossy rock''}\vspace{1mm}\\
      \HEXAND{FLUX-OTF}{10}\\
      \multicolumn{6}{c}{\itshape``A peacock, a fountain pen, and a silk scarf on a marble table''}\vspace{1mm}\\
      \HEXAND{FLUX-OTF}{0}\\
      \multicolumn{6}{c}{\itshape``A hammerhead shark and a great white shark circling over a coral shelf''}\vspace{1mm}\\
      \HEXAND{FLUX-OTF}{9}\\
      \multicolumn{6}{c}{\itshape``A windmill, a picnic blanket, and a bicycle with a basket of tulips''}\vspace{1mm}\\
      \HEXAND{FLUX-OTF}{5}\\
      \multicolumn{6}{c}{\itshape``A quartz crystal, an amethyst, and a citrine displayed on black velvet''}\vspace{1mm}\\
      \HEXAND{FLUX-OTF}{3}\\
      \multicolumn{6}{c}{\itshape``A chef’s knife, a santoku, and a paring knife laid on a cutting board''}\vspace{1mm}\\
    \end{tabular}%
  }}
  \caption{FLUX.1 [dev] samples with test-time control. All evaluated heuristics are shown at their optimal $\lambda$.}
  \label{fig:otf_flux}
\end{figure}

\subsection{Test-Time Control: Stable Diffusion 1.5}
\begin{figure}[H]
  \centering
  {\setlength{\tabcolsep}{.3mm}%
  \resizebox*{!}{\dimexpr0.95\textheight-2\baselineskip\relax}{%
    \begin{tabular}{@{}ccccccc@{}}
      \textbf{Base} & \textbf{Attend\&Excite} & \textbf{CONFORM} & \textbf{Divide\&Bind} & \textbf{FOCUS (Ours)} \\
      \QUINTAND{SD1-OTF}{1}\\
      \multicolumn{5}{c}{\itshape``A red fox and an arctic fox sitting side by side in tall grass''}\vspace{1mm}\\
      \QUINTAND{SD1-OTF}{7}\\
      \multicolumn{5}{c}{\itshape``A snowboard, a telescope, and a husky on a snowy ridge''}\vspace{1mm}\\
      \QUINTAND{SD1-OTF}{0}\\
      \multicolumn{5}{c}{\itshape``A black bear and a brown bear ambling along a riverbank''}\vspace{1mm}\\
      \QUINTAND{SD1-OTF}{4}\\
      \multicolumn{5}{c}{\itshape``A barn owl, a snowy owl, and a great horned owl perched in a rustic loft''}\vspace{1mm}\\
      \QUINTAND{SD1-OTF}{8}\\
      \multicolumn{5}{c}{\itshape``A cello, a bonsai, and a ceramic teacup in a quiet teahouse''}\vspace{1mm}\\
      \QUINTAND{SD1-OTF}{3}\\
      \multicolumn{5}{c}{\itshape``A puffin and a penguin standing on a windswept shoreline''}\vspace{1mm}\\
    \end{tabular}%
  }}
  \caption{Stable Diffusion 1.5 samples with test-time control. All evaluated heuristics are shown at their optimal $\lambda$.}
  \label{fig:otf_sd1}
\end{figure}

\subsection{Fine-tuned: Stable Diffusion 3.5}
\begin{figure}[H]
  \centering
  {\setlength{\tabcolsep}{.3mm}%
  \resizebox{0.98\linewidth}{!}{%
    \begin{tabular}{@{}ccccccc@{}}
      \textbf{Base} & \textbf{Attend\&Excite} & \textbf{CONFORM} & \textbf{Divide 
      \&Bind} & \textbf{Self-Cross Guidance} & \textbf{FOCUS (Ours)} \\
      \HEXAND{SD3-FINE}{1}\\
      \multicolumn{6}{c}{\itshape``A Siberian Husky, an Alaskan Malamute, and a Samoyed trotting through fresh snow''}\vspace{1mm}\\
      \HEXAND{SD3-FINE}{15}\\
      \multicolumn{6}{c}{\itshape``A magician, a white rabbit, and a deck of cards on a velvet stage''}\vspace{1mm}\\
      \HEXAND{SD3-FINE}{22}\\
      \multicolumn{6}{c}{\itshape``A horse and a bear in a forest''}\vspace{1mm}\\
      \HEXAND{SD3-FINE}{0}\\
      \multicolumn{6}{c}{\itshape``A robin, a bluebird, and a warbler perched on a garden fence at dawn''}\vspace{1mm}\\
      \HEXAND{SD3-FINE}{16}\\
      \multicolumn{6}{c}{\itshape``A painter, a foxglove, and an easel by a cliffside path''}\vspace{1mm}\\
      \HEXAND{SD3-FINE}{21}\\
      \multicolumn{6}{c}{\itshape``A black cat, an orange cat, and a white cat lounging on a windowsill''}\vspace{1mm}\\
    \end{tabular}%
  }}
  \caption{Sample results from Stable Diffusion~3.5 fine-tuned with each heuristic. Prompts were not seen during training to evaluate generalization. All images are generated with identical settings and each heuristic is shown at its optimal trained $\lambda$.}
  \label{fig:fine_sd3}
\end{figure}

\subsection{Fine-tuned: FLUX.1 [dev]}
\begin{figure}[htb]
  \centering
  {\setlength{\tabcolsep}{.3mm}%
  \resizebox{0.98\linewidth}{!}{%
    \begin{tabular}{@{}ccccccc@{}}
      \textbf{Base} & \textbf{Attend\&Excite} & \textbf{CONFORM} & \textbf{Divide 
      \&Bind} & \textbf{Self-Cross Guidance} & \textbf{FOCUS (Ours)} \\
      \HEXAND{FLUX-FINE}{1}\\
      \multicolumn{6}{c}{\itshape``A red fox and an arctic fox sitting side by side in tall grass''}\vspace{1mm}\\
      \HEXAND{FLUX-FINE}{24}\\
      \multicolumn{6}{c}{\itshape``A mooncake, a teapot, and a jade rabbit under paper lanterns''}\vspace{1mm}\\
      \HEXAND{FLUX-FINE}{26}\\
      \multicolumn{6}{c}{\itshape``A lighthouse, a cello, and a red scarf beside crashing waves''}\vspace{1mm}\\
      \HEXAND{FLUX-FINE}{10}\\
      \multicolumn{6}{c}{\itshape``A lynx, a bobcat, and a cougar stepping across a rocky ledge''}\vspace{1mm}\\
      \HEXAND{FLUX-FINE}{21}\\
      \multicolumn{6}{c}{\itshape``A bluetit, a croissant, and a porcelain cup on a balcony rail''}\vspace{1mm}\\
      \HEXAND{FLUX-FINE}{23}\\
      \multicolumn{6}{c}{\itshape``A jellyfish, a seashell, and a glass bottle drifting in turquoise water''}\vspace{1mm}\\
    \end{tabular}%
  }}
  \caption{Sample results from FLUX.1 [dev] fine-tuned with each heuristic. Prompts were not seen during training to evaluate generalization. All images are generated with identical settings and each heuristic is shown at its optimal trained $\lambda$.}
  \label{fig:fine_flux}
\end{figure}

\section{Detailed Evaluation Results}\label{sec:detailed_results}
\subsection{Evaluation Results for Test-Time Control: Stable Diffusion 3.5}
\begin{center}
    \captionsetup{hypcap=false}
    \captionof{table}{Test-time control results for each heuristic on Stable Diffusion 3.5. We report mean $\pm$ std over all prompts and seeds. All methods use the same sampling and evaluation pipeline. Values are color-coded relative to the base case: entries close to the base value are white, improvements are shown in blue and degradations in red, with stronger color intensity indicating larger deviations from the base.}
    \vspace{-1.5mm}
    \label{tab:metrics_otf_sd3}
    \centering
    \resizebox{\textwidth}{!}{\renewcommand{\arraystretch}{0.95}\setlength{\tabcolsep}{1.5mm}\tableOtfSDIII}
\end{center}
\clearpage

\subsection{Evaluation Results for Test-Time Control: FLUX.1 [dev]}
\begin{center}
    \captionsetup{hypcap=false}
    \captionof{table}{Test-time control results for each heuristic on FLUX.1 [Dev]. We report mean $\pm$ std over all prompts and seeds. All methods use the same sampling and evaluation pipeline. Values are color-coded relative to the base case: entries close to the base value are white, improvements are shown in blue and degradations in red, with stronger color intensity indicating larger deviations from the base.}
    \vspace{-1.5mm}
    \label{tab:metrics_otf_flux}
    \centering
    \resizebox{\textwidth}{!}{\renewcommand{\arraystretch}{0.95}\setlength{\tabcolsep}{1.5mm}\tableOtfFLUX}
\end{center}
\clearpage

\subsection{Evaluation Results for Test-Time Control: Stable Diffusion 1.5}
\begin{center}
    \captionsetup{hypcap=false}
    \captionof{table}{Test-time control results for each heuristic on Stable Diffusion 1.5. We report mean $\pm$ std over all prompts and seeds. All methods use the same sampling and evaluation pipeline. Values are color-coded relative to the base case: entries close to the base value are white, improvements are shown in blue and degradations in red, with stronger color intensity indicating larger deviations from the base.}
    \vspace{-1.5mm}
    \label{tab:metrics_otf_sd1}
    \centering
    \resizebox{\textwidth}{!}{\renewcommand{\arraystretch}{0.95}\setlength{\tabcolsep}{1.5mm}\tableOtfSDI}
\end{center}
\clearpage

\subsection{Evaluation Results for Fine-Tuning: Stable Diffusion 3.5}
\begin{center}
    \captionsetup{hypcap=false}
    \captionof{table}{\textbf{Part I}: Fine-tuning results for each heuristic on Stable Diffusion 3.5 across different hyperparameter configurations. Here, $N$ denotes the number of prompts in the dataset, \textit{Lr} the learning rate, \textit{Lambda} the scalar applied to the heuristic function, and \textit{Ckpt.} the checkpoint used for evaluation. We report mean $\pm$ std over all prompts and seeds. All methods use the same sampling and evaluation pipeline. Values are color-coded relative to the base case: entries close to the base value are white, improvements are shown in blue and degradations in red, with stronger color intensity indicating larger deviations from the base.}
    \vspace{-1.5mm}
    \label{tab:metrics_fine_sd3_1}
    \centering
    \resizebox{\textwidth}{!}{\renewcommand{\arraystretch}{1.05}\setlength{\tabcolsep}{1.5mm}\tableFINESDI}
\end{center}

\begin{center}
    \captionsetup{hypcap=false}
    \captionof{table}{\textbf{Part II}: Fine-tuning results for each heuristic on Stable Diffusion 3.5 across different hyperparameter configurations. Here, $N$ denotes the number of prompts in the dataset, \textit{Lr} the learning rate, \textit{Lambda} the scalar applied to the heuristic function, and \textit{Ckpt.} the checkpoint used for evaluation. We report mean $\pm$ std over all prompts and seeds. All methods use the same sampling and evaluation pipeline. Values are color-coded relative to the base case: entries close to the base value are white, improvements are shown in blue and degradations in red, with stronger color intensity indicating larger deviations from the base.}
    \vspace{-1.5mm}
    \label{tab:metrics_fine_sd3_2}
    \centering
    \resizebox{\textwidth}{!}{\renewcommand{\arraystretch}{1.05}\setlength{\tabcolsep}{1.5mm}\tableFINESDII}
\end{center}

\begin{center}
    \captionsetup{hypcap=false}
    \captionof{table}{\textbf{Part III}: Fine-tuning results for each heuristic on Stable Diffusion 3.5 across different hyperparameter configurations. Here, $N$ denotes the number of prompts in the dataset, \textit{Lr} the learning rate, \textit{Lambda} the scalar applied to the heuristic function, and \textit{Ckpt.} the checkpoint used for evaluation. We report mean $\pm$ std over all prompts and seeds. All methods use the same sampling and evaluation pipeline. Values are color-coded relative to the base case: entries close to the base value are white, improvements are shown in blue and degradations in red, with stronger color intensity indicating larger deviations from the base.}
    \vspace{-1.5mm}
    \label{tab:metrics_fine_sd3_3}
    \centering
    \resizebox{\textwidth}{!}{\renewcommand{\arraystretch}{1.05}\setlength{\tabcolsep}{1.5mm}\tableFINESDIII}
\end{center}
\clearpage

\subsection{Evaluation Results for Fine-Tuning: FLUX.1 [dev]}
\begin{center}
    \captionsetup{hypcap=false}
    \captionof{table}{\textbf{Part I}: Fine-tuning results for each heuristic on FLUX.1 [dev] across different hyperparameter configurations. Here, $N$ denotes the number of prompts in the dataset, \textit{Lr} the learning rate, \textit{Lambda} the scalar applied to the heuristic function, and \textit{Ckpt.} the checkpoint used for evaluation. We report mean $\pm$ std over all prompts and seeds. All methods use the same sampling and evaluation pipeline. Values are color-coded relative to the base case: entries close to the base value are white, improvements are shown in blue and degradations in red, with stronger color intensity indicating larger deviations from the base.}
    \vspace{-1.5mm}
    \label{tab:metrics_fine_flux_1}
    \centering
    \resizebox{\textwidth}{!}{\renewcommand{\arraystretch}{1.05}\setlength{\tabcolsep}{1.5mm}\tableFINEFLUXI}
\end{center}

\begin{center}
    \captionsetup{hypcap=false}
    \captionof{table}{\textbf{Part II}: Fine-tuning results for each heuristic on FLUX.1 [dev] across different hyperparameter configurations. Here, $N$ denotes the number of prompts in the dataset, \textit{Lr} the learning rate, \textit{Lambda} the scalar applied to the heuristic function, and \textit{Ckpt.} the checkpoint used for evaluation. We report mean $\pm$ std over all prompts and seeds. All methods use the same sampling and evaluation pipeline. Values are color-coded relative to the base case: entries close to the base value are white, improvements are shown in blue and degradations in red, with stronger color intensity indicating larger deviations from the base.}
    \vspace{-1.5mm}
    \label{tab:metrics_fine_flux_2}
    \centering
    \resizebox{\textwidth}{!}{\renewcommand{\arraystretch}{1.05}\setlength{\tabcolsep}{1.5mm}\tableFINEFLUXII}
\end{center}

\begin{center}
    \captionsetup{hypcap=false}
    \captionof{table}{\textbf{Part III}: Fine-tuning results for each heuristic on FLUX.1 [dev] across different hyperparameter configurations. Here, $N$ denotes the number of prompts in the dataset, \textit{Lr} the learning rate, \textit{Lambda} the scalar applied to the heuristic function, and \textit{Ckpt.} the checkpoint used for evaluation. We report mean $\pm$ std over all prompts and seeds. All methods use the same sampling and evaluation pipeline. Values are color-coded relative to the base case: entries close to the base value are white, improvements are shown in blue and degradations in red, with stronger color intensity indicating larger deviations from the base.}
    \label{tab:metrics_fine_flux_3}
    \vspace{-1.5mm}
    \resizebox{\textwidth}{!}{\renewcommand{\arraystretch}{1.05}\setlength{\tabcolsep}{1.5mm}\tableFINEFLUXIII}
\end{center}